\documentclass[letterpaper]{article} % DO NOT CHANGE THIS
\RequirePackage{etex}
\usepackage[draft]{aaai25}  % DO NOT CHANGE THIS
\usepackage{times}  % DO NOT CHANGE THIS
\usepackage{helvet}  % DO NOT CHANGE THIS
\usepackage{courier}  % DO NOT CHANGE THIS
\usepackage[hyphens]{url}  % DO NOT CHANGE THIS
\usepackage{graphicx} % DO NOT CHANGE THIS
\usepackage{natbib}  % DO NOT CHANGE THIS AND DO NOT ADD ANY OPTIONS TO IT
\usepackage{caption} % DO NOT CHANGE THIS AND DO NOT ADD ANY OPTIONS TO IT
\title{When to Trust Your Data: Enhancing Dyna-Style Model-Based Reinforcement Learning With Data Filter}
\author{
    Yansong Li\equalcontrib\textsuperscript{\rm $\P$}\footnote{ Corresponding author},
    Zeyu Dong\equalcontrib\textsuperscript{\rm $\ddag$},
    Ertai Luo\textsuperscript{\rm $\ddag$},
    Yu Wu\textsuperscript{\rm $\S$},
    Shuo Wu\textsuperscript{\rm $\P$},
    Shuo Han\textsuperscript{\rm $\P$}
}
\affiliations{
    \textsuperscript{\rm $\P$}University of Illinois Chicago\\
    \texttt{yli340, swu99, hanshuo@uic.edu}\\
    \textsuperscript{\rm $\ddag$}Stony Brook University\\
    \texttt{Zeyu.Dong@stonybrook.edu, erluo@cs.stonybrook.edu}\\
    \textsuperscript{\rm $\S$}Rutgers University\\
    \texttt{yw828@scarletmail.rutgers.edu}
}

\usepackage{tikz}
\usetikzlibrary{patterns}

\usepackage{multirow}
\usepackage{booktabs}
\usepackage{amsmath}
\usepackage{amssymb}
\usepackage{mathtools}
\usepackage{cleveref}
\usepackage{amsthm}
\theoremstyle{plain}
\newtheorem{theorem}{Theorem}[section]
\newtheorem{proposition}[theorem]{Proposition}

\theoremstyle{definition}
\newtheorem{definition}[theorem]{Definition}
\newtheorem{assumption}[theorem]{Assumption}
\theoremstyle{remark}

\usepackage[T1]{fontenc}
\usepackage[utf8]{inputenc}
\usepackage[font=small]{caption}
\usepackage[font=scriptsize]{subcaption}
\usepackage[english]{babel}
\usepackage{xcolor}
\definecolor{darkgreen}{rgb}{0,0.5,0}

\usepackage{algpseudocode}
\usepackage{algorithm}

\renewcommand{\algorithmicrequire}{\textbf{Input:}}
\renewcommand{\algorithmicensure}{\textbf{Output:}}
\usepackage{mathtools}
\usepackage{autonum}
\makeatletter
\newcommand{\restore@Environment}[1]{%
    \AtBeginDocument{%
        \csletcs{#1*}{#1}%
        \csletcs{end#1*}{end#1}%
    }%
}
\forcsvlist\restore@Environment{alignat,equation,gather,multline,flalign,align}
\makeatother

\usepackage{acro}
\DeclareAcronym{mbpo}{
    short = MBPO,
    long = model-based policy optimization,
}
\DeclareAcronym{sac}{
    short = SAC,
    long = soft actor-critic
}

\DeclareAcronym{knn}{
    short = KNN,
    long = K-Nearest Neighbors
}

\DeclareAcronym{mdp}{
    short = MDP,
    long = Markov decision process
}

\DeclareAcronym{DQN}{
    short = DQN,
    long = Deep Q-Network
}
\DeclareAcronym{ood}{
    short = OOD,
    long = Out-of-distribution 
}

\DeclareAcronym{sota}{
    short = SOTA,
    long = state-of-the-art
}

\DeclareAcronym{ann}{
    short = ANN,
    long = Approximate Nearest Neighbour 
}

\newcommand{\pp}{\mathbb{P}}
\newcommand{\pps}{\mathbb{P}^{*}}
\newcommand{\hpp}{\hat{\mathbb{P}}}
\newcommand{\hp}{\hat{\mathbb{P}}}
\newcommand{\ee}{\mathbb{E}}
\newcommand{\dreal}{\mathcal{D}_{\mathrm{real}}}
\newcommand{\dest}{\mathcal{D}_{\mathrm{est}}}
\newcommand{\dreduct}{\mathcal{D}_{\mathrm{reduct}}}

\newcommand{\policy}{\pi_{\phi}}

\newcommand{\hs}{\hat{s}}
\newcommand{\ha}{\hat{a}}
\newcommand{\muhat}{\hat{\mu}}
\begin{document}

\maketitle

\begin{abstract}

    % Reinforcement learning (RL) algorithms can be divided into two classes: model-free algorithms, which are sample inefficient, and model-based algorithms, which suffer from model bias.
    % %
    % Dyna-style algorithms combine these two approaches by using simulated data from an estimated environmental model to accelerate model-free training.
    % %
    % However, their efficiency is compromised when the estimated model is inaccurate. Previous works address this issue by using model ensembles or pretraining the estimated model with data collected from the real environment, increasing computational and sample complexity.
    % %
    % To tackle this issue, we introduce an \ac{ood} data filter that removes simulated data from the estimated model that significantly diverges from data collected in the real environment.
    % %
    % We show theoretically that this technique enhances the quality of simulated data. 
    % %
    % With the help of the OOD data filter, the data simulated from the estimated model better mimics the data collected by interacting with the real model. This improvement is evident in the critic updates compared to using the simulated data without the OOD data filter.
    % %
    % Our experiment integrates the data filter into the \ac{mbpo} algorithm. The results demonstrate that our method requires fewer interactions with the real environment to achieve higher level of optimality as \ac{mbpo}, even without model ensemble.

    Reinforcement learning (RL) algorithms can be divided into two classes: model-free algorithms, which are sample-inefficient, and model-based algorithms, which suffer from model bias.
    Dyna-style algorithms combine these two approaches by using simulated data from an estimated environmental model to accelerate model-free training.
    However, their efficiency is compromised when the estimated model is inaccurate. Previous works address this issue by using model ensembles or pretraining the estimated model with data collected from the real environment, increasing computational and sample complexity.
    To tackle this issue, we introduce an \ac{ood} data filter that removes simulated data from the estimated model that significantly diverges from data collected in the real environment.
    We show theoretically that this technique enhances the quality of simulated data. 
    With the help of the OOD data filter, the data simulated from the estimated model better mimics the data collected by interacting with the real model. This improvement is evident in the critic updates compared to using the simulated data without the OOD data filter.
    Our experiment integrates the data filter into the \ac{mbpo} algorithm. The results demonstrate that our method requires fewer interactions with the real environment to achieve a higher level of optimality than \ac{mbpo}, even without a model ensemble.

    % Reinforcement learning (RL) algorithms can be divided into two classes: model-free algorithms, which are sample inefficient, and model-based algorithms, which suffer from model bias. Dyna-style algorithms combine these two approaches by using simulated data from an estimated environmental model to accelerate model-free training. However, their efficiency is compromised when the estimated model is inaccurate. Previous works address this issue by using model ensembles or pretraining the estimated model with data collected from the real environment, increasing computational and sample complexity. To tackle this issue, we introduce an out-of-distribution (OOD) data filter that removes simulated data from the estimated model that significantly diverges from data collected in the real environment. We show theoretically that this technique enhances the quality of simulated data. With the help of the OOD data filter, the data simulated from the estimated model better mimics the data collected by interacting with the real model. This improvement is evident in the critic updates compared to using the simulated data without the OOD data filter. Our experiment integrates the data filter into the model-based policy optimization (MBPO) algorithm. The results demonstrate that our method requires fewer interactions with the real environment to achieve higher level of optimality as MBPO, even without model ensemble.
\end{abstract}

\section{Introduction}

Reinforcement learning (RL)~\citep{sutton_reinforcement_2018,li_deep_2018} algorithms can be classified into model-free and model-based algorithms.
Model-free algorithms~\citep{mnih_human-level_2015,haarnoja_soft_2018} do not maintain an explicit form of the environment dynamics and can be applied in environments whose dynamics are hard to model.
However, one drawback of model-free algorithms is that they require a large number of interactions with the environment during training, which can be impractical in physical systems, where each interaction takes place in the real world.

In comparison, model-based algorithms~\citep{doll_ubiquity_2012,wang_benchmarking_2019,chua_deep_2018,deisenroth_pilco_2011} maintain an estimated model of the environment dynamics, such as neural networks for the transition kernel and the reward function. Traditional model-based algorithms allow for planning based on the estimated model, though they may suffer from model bias, leading to policies that are far from optimal. Without model bias, model-based methods can find optimal policies with fewer interactions with the environment.

Besides traditional model-based algorithms, Dyna-style algorithms~\citep{sutton_dyna_1991, janner_when_2019, wang_benchmarking_2019} also maintain an estimated model. Instead of relying on planning with the estimated model, Dyna-style algorithms use data simulated from the estimated model to accelerate model-free methods while updating the estimated model. This approach reduces the impact of bias in the estimated model. One example of Dyna-style algorithms is \acl{mbpo} (MBPO)~\citep{janner_when_2019}, which estimates the model of the environment and simulates data from the estimated model to speed up \ac{sac}~\citep{haarnoja_soft_2018}.

One weakness of Dyna-style algorithms is that some data simulated from the estimated model do not resemble those obtained from interactions with the real world. Using such data may instead slow down model-free training. Previous works address this issue either by collecting a large amount of data through random interactions with the real environment to pre-train the estimated model before model-free training, which is costly in the real world due to the numerous interactions required, or by simply reducing the rollout length, which is conservative. In fact, \citet{janner_when_2019} and \citet{sutton_dyna_1991} use these two strategies together. Additionally, to reduce the variance of the data simulated from estimated models, \citet{janner_when_2019} used an ensemble of models, which is computationally expensive.

In this paper, we tackle this problem by introducing a method called \emph{out-of-distribution data filter}. In a nutshell, our method filters out data points simulated from the estimated model if the data points are far from those interacted with the real model before sending the data for model-free training. 
This method can be regarded as an instance of \ac{ood} detection~\citep{nasvytis_rethinking_2024,liu_towards_2023,yang_generalized_2024,hsu_generalized_2020} with the real distribution being the transition kernel and the exploration policy of the real environment. We theoretically show that the out-of-distribution data filter not only lowers the drift of the rollout trajectory but also lowers the parametric drift of \(Q\)-network updates. Thus, with the help of the \ac{ood} data filter, the data simulated from the estimated model better mimics the data collected by interacting with the real model in terms of \(Q\)-network updates than it would without the \ac{ood} data filter.

Our experiments are conducted in the \texttt{MuJoCo}~\citep{todorov_mujoco_2012} environment.
The results show that when applying the out-of-distribution data filter to \ac{mbpo}, it achieves \ac{sota} performance even without using a model ensemble.

\section{Related work}

\emph{Dyna-style algorithms}: An early example of a Dyna-style algorithm is Dyna-\(Q\)~\citep{sutton_dyna_1991}, which estimates the transition kernel during
\(Q\)-learning and uses data simulated from the estimated model to accelerate the training of model-free \(Q\)-learning. In deep reinforcement learning, \ac{mbpo}~\citep{janner_when_2019} employs an ensemble of neural networks to estimate both the transition kernel and the reward function. The algorithm then uses simulated data from the estimated model to expedite the training of \ac{sac}~\citep{haarnoja_soft_2018}. More recently, \citet{zheng_is_2022} modify the \ac{mbpo} algorithm by introducing an extra critic regularization term in the training of \ac{sac}, and their method achieves similar performance to \ac{mbpo} in \texttt{MuJoCo}~\citep{todorov_mujoco_2012,brockman_openai_2016} environment without model ensemble.

Besides \ac{mbpo}, earlier works that apply Dyna-style techniques to deep RL include the Model-Ensemble Trust-Region Policy Optimization \citep{kurutach_model-ensemble_2018} and its modifications \citep{clavera_model-based_2018,luo_algorithmic_2021}. \citet{luo_algorithmic_2021} provide a theoretical guarantee of monotonic improvement. However, their experiments on \texttt{MuJoCo} do not demonstrate state-of-the-art performance, so we did not include their algorithms in our experiments.

Several works explore the trade-off between data simulated from the estimated model and the data interacted with the real model; see~\citet{lai_effective_2022,shen_adaptation_2023} for example.
In this work, we introduced the out-of-distribution data filter. We achieved SOTA performance (requiring less interaction to achieve the same level of suboptimality) compared to \ac{mbpo} on the \texttt{MuJoCo} environment. Also, to the best of our knowledge, we are the first to analyze when Dyna-style algorithms perform well from an out-of-distribution perspective. 

\emph{Model-based reinforcement learning}: Dyna-style algorithms are a subset of model-based reinforcement learning algorithms. In addition to Dyna-style methods, there are other approaches such as PILCO~\citep{deisenroth_pilco_2011}, iLQG~\citep{tassa_synthesis_2012-1}, and Guided Policy Search~\citep{levine_learning_2014,levine_learning_2015,zhang_learning_2016,finn_guided_2016}. \citet{wang_benchmarking_2019} provides a review that summarizes commonly used model-based RL algorithms.

\emph{Out-of-distribution detection}: Out-of-distribution detection~\citep{hsu_generalized_2020,yang_generalized_2024,sedlmeier_uncertainty-based_2019} is a technique used in machine learning to identify data points that deviate from the expected distribution. Recently, \citet{nasvytis_rethinking_2024} introduced several out-of-distribution concepts to reinforcement learning. Specifically, the \ac{ood} problem we focus on includes sensory and semantic anomaly detection, as defined in \citet{nasvytis_rethinking_2024}, which are used to determine whether states are generated from the real environmental model.

\emph{Biased gradient learning}: In Dyna-style algorithms, data simulated from the estimated model is sent to the model-free algorithm for training. The gradient used for the update rule contains the error that is derived from the estimation error of the estimated model. Based on that, the model-free training can be reformulated as a Biased Gradient Learning problem. Previous works~\citep{ajalloeian_convergence_2021,wang_finite-time_2020,li_accelerating_2022} considered the additive noise and showed error bounds in gradient-based optimization methods and the convergence rate when error is bounded under certain conditions. In this work, we relax the requirement of additive noise and prove that the shifting of \( Q \)-network updates is upper bounded after each update with smoothness assumptions.

\section{Preliminary}

A \ac{mdp} is a tuple $(S, A, \pp, r, \gamma, \mu)$ where \( S \) is the state space, \( A \) is the action space, \( \pp \colon S \times A \to \Delta(S) \) is the transition kernel, \( r \colon S \times A \to \mathbb{R}\) is the reward function, and \( \gamma \in [0,1] \) is the discount factor. The term \( \Delta(S) \) represents the set of all probability distributions over the state space \( S \). A parameterized policy \( \pi_\phi \) parameterized by \( \phi \) is defined as \( \policy \colon S \to \Delta(A) \).
The initial state \(s_1\) is sampled from a distribution \(\mu\).
We abuse the notation of \( \pp \) and define the transition of a policy as
\begin{equation}
    \pp (s, \policy) \triangleq \ee_{a \sim \pi_\phi(s)}  \pp(s,a).
\end{equation}
Throughout this paper, the term \emph{model} refers to the transition kernel. The reward function is known to the learner. One can easily generalize our result to the case that the model refers to both the transition kernel and the reward function, see~\citet{janner_when_2019} for example.
In the following analysis, we consider an \ac{mdp} with continuous state and action spaces. 
The transition kernel is assumed to be Gaussian and defined as
\begin{equation}
    \mathbb{P} (s, a) \triangleq (\mu (s, a),
    \sigma^2 (s, a))^\top
\end{equation}
and
\[
    \mathbb{P} (s, \pi_{\phi} (s))
    \triangleq
    \mathbb{E}_{a \sim \pi_{\phi} (s)} \left[(\mu (s, a),
        \sigma^2 (s, a))^\top \right].
\]
We abuse the notation \( \pp(s,a) \) to represent the Gaussian distribution \( \mathcal{N} (\mu (s, a), \sigma^2 (s, a)) \) with mean \( \mu (s, a) \) and variance \( \sigma^2 (s, a) \). Assuming that the transition distribution is Gaussian is common in continuous state \ac{mdp}s~\citep{deisenroth_pilco_2011,janner_when_2019} and robotics~\citep{peters_reinforcement_2003,kober_reinforcement_2014,brunke_safe_2021,johannink_residual_2019}

Also, we use the following notations:
\begin{equation}
    \mathbb{E}_{a \sim \pi_{\phi} (s)} \mu (s, a)
    \triangleq
    \mu_{\pi_{\phi} (s)} (s)
\end{equation}
and
\begin{equation}
    \mathbb{E}_{a \sim \pi_{\phi} (s)} \sigma^2 (s, a)\triangleq
    \sigma^2_{\pi_{\phi}(s)} (s).
\end{equation}
Thus,
\[
    s' \sim \mathbb{P} (s, \pi_{\phi} (s)) \Leftrightarrow s' \sim \mathcal{N}
    (\mu_{\pi_{\phi} (s)} (s), \sigma^2_{\pi_{\phi} (s)} (s)) .
\]
For simplicity, our theoretical results are derived using a single-variable Gaussian distribution. These results can be easily generalized to multivariate Gaussian distributions by replacing the mean with a vector and the variance with a covariance matrix. See \citet{deisenroth_pilco_2011} and \citet{janner_when_2019} for examples.

\subsection{Dyna-style algorithm}

In this subsection, we introduce Dyna-style algorithms, the idea of Dyna-style algorithms is summarized in Algorithm~\ref{algo:modelbasedaddmodelfree}.

\begin{algorithm}
    \caption{\textsc{Dyna-Style learning}}
    \label{algo:modelbasedaddmodelfree}
    \begin{algorithmic}[1]
        \Require pretrained model \( \hat{\pp}^1 \), total number of episodes \( K \), rollout length \(L\), batch size \(N\), rollout size \( M = N L \), model estimator \( \texttt{ModelEstimator} \), time horizon \( H \), model-free algorithm \texttt{ModelFreeAlgo}.
        \Ensure \( \pi_{\phi}\)
        \State \( \dreal \gets \emptyset  \), \( \dest \gets \emptyset \), \(\hp \gets \hp^1\)
        \For{\( k \in [K] \)}
        \State \(s_1 \sim \mu\)
        \For{\( h \in [H] \)}
        \State uniformly sample states from \(\dreal\), perform \(L\)-step rollout with \(\hp\) under policy \(\policy\), repeat \(N\) times until \(M \) data points are collected. Store the data in \(\dest\)   \label{algo:collect_data_est}
        \State \( \phi \gets \texttt{ModelFreeAlgo}(\phi, \dreal, \dest) \)
        \State sample \(a_h\) from \(\policy(s_h)\) and sample \(s_{h+1}\) from \(\pp^*(s_h,a_h)\), add \( (s_h,a_h,s_{h+1})\) to \( \dreal \).
        \State \( \hpp \gets \texttt{ModelEstimator}(\hpp , \dreal) \)  \label{algo:collect_data_real}
        \EndFor
        \EndFor
        \State \Return \( \pi_{\phi} \)
    \end{algorithmic}
\end{algorithm}
The main idea of Algorithm~\ref{algo:modelbasedaddmodelfree} is to use the data \( \dest \) simulated with the estimated model together with the data \( \dreal \) interacted from the real model to speed up the model-free algorithm \( \texttt{ModelFreeAlgo} \). The model-free algorithm \( \texttt{ModelFreeAlgo} \) can be \( Q \)-learning as in Dyna-\( Q\)~\citep{sutton_dyna_1991} or \ac{sac}~\citep{haarnoja_soft_2018} as in \ac{mbpo}~\citep{janner_when_2019}.

The model of the real environment is denoted by \( \pps \). The estimated model is denoted as \( \hpp \). The estimated model is updated by \( \texttt{ModelEstimator} \) using data \( \dreal \) collected by interacting with the real model \(\pp^*\). The \( \texttt{ModelEstimator} \) method used by Dyna-\( Q \) and \ac{mbpo} is the maximum likelihood estimator.
Also, Line~\ref{algo:collect_data_real} updates the estimated model at every step. One can also modify this so the model can be updated less frequently; see Algorithm 2 in \citet{liu_re_2019} for example.

We denote the pretrained estimated model as \( \hpp^1 \). The pre-trained model is trained by pre-collected data under \( \pps \). The pre-collected data \( (s,a,s') \) to train \(\hp^1\) is collected by randomly sampling \( (s,a) \) from \( S \times A \) and sample the next state \( s' \) from the real model, i.e., \( s' \sim \pps(s,a) \).
We use \( (s_h,a_h,s'_h = s_{h + 1}) \) to represent a data point in \(\dreal\) that is collected by interacting with the real model \(\pp^*\) in the \( h \)-th step. 

The set of data points simulated from the estimtated model is denoted by \(\dest\). Each data point is a tuple consisting of a state, an action, and the next state. The simulation method is given by  Line~\ref{algo:collect_data_est}. Line~\ref{algo:collect_data_est} is a sample strategy named \emph{branched rollout}.  As introduced by \citet{sutton_dyna_1991}, each round (we sample \(N\) rounds) of branched rollout takes a state from \(\dreal\) and simulates a trajectory based on \(\hp\). The blue lines in Fig.~\ref{fig:OOD data filter} visualize several rollouts branched from red lines, where points in red lines are taken from \(\dreal\).

What's more, Algorithm~\ref{algo:modelbasedaddmodelfree} is an on-policy algorithm, meaning the policy used to collect data is the policy we want to train. One can easily modify the algorithm to be off-policy by replacing the policy in Line~\ref{algo:collect_data_est} and Line~\ref{algo:collect_data_real} with an exploration policy, such as \( \epsilon \)-greedy policy.

Algorithm~\ref{algo:modelbasedaddmodelfree} only maintains a single estimated model. In contrast, the original \ac{mbpo} paper maintains an ensemble of models to reduce the variance of data simulated from estimated models, which increases the computational complexity of the algorithm. 

One issue of previous Dyna-style algorithms such as \ac{mbpo} is that these algorithms randomly select data points that are simulated from the estimated model without taking the quality of data into consideration. If a data point is far away from the data interacted with the real environment when the estimated model is inaccurate, the estimation error is not negligible and will be propagated to the next data point. Eventually, the error stemming from the estimated model will undermine the model-free training. 
As a consequence, the rollout length \(L\) of \ac{mbpo} algorithm must be short (it is set to be one in most cases, in fact). According to numerical observations, a larger rollout length \(L\) causes larger error. Fig.~\ref{fig:long_rollout} shows the experiment conducted in the \texttt{InvertedPendulum-V2} environment. 
In this experiment, we fix the rollout size \(M\) and chose two different rollout lengths: \(L = 1\) and \(L = 10\). All other parameters are the same as in \citet{janner_when_2019}. As shown in Fig.~\ref{fig:long_rollout}, the \ac{mbpo} with a larger rollout length \(L = 10\) requires more interactions to reach the same level of optimality.

\begin{figure}[t]
    \centering
    \includegraphics[width=1\columnwidth]{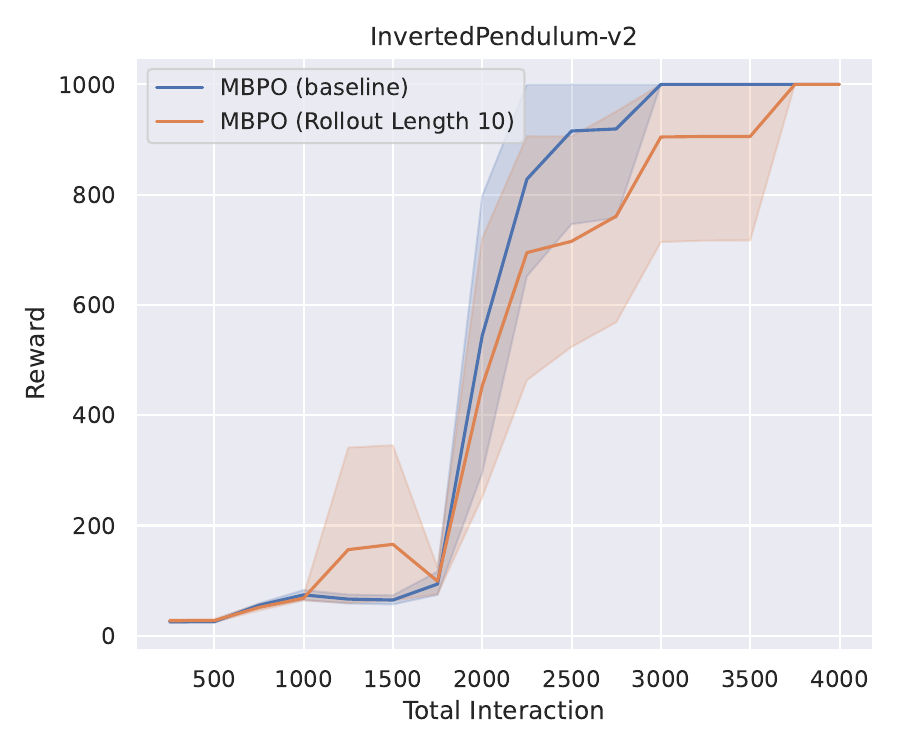}
    \caption{Error accumulation for larger rollout length in \ac{mbpo}}
    \label{fig:long_rollout}
\end{figure}

\section{Out-of-distribution data filter}

In Dyna-style algorithms, data points on \( \dreal \) are generated by \( \policy \) and \( \pp^* \). To enhance the performance of Dyna-style algorithms, our approach is to filter out data points in \( \dest \) that are not likely to be generated by \( \policy \) and \( \pp^* \). Based on these intuitions, we introduce our Out of Distribution Data Filter in Algorithm~\ref{algo:DataFilter}. 

\begin{algorithm}
    \caption{\textsc{OOD-DataFilter}}
    \label{algo:DataFilter}
    \begin{algorithmic}[1]
        \Require \( \dreal \), \( \dest \), reject level \( \epsilon_k \)
        \Ensure \( \dreduct \)
        \State \( \dreduct \gets \dest \)
        \For{\( (s,a,s') \in \dest \)}
        \If{ \(\forall (s^*, a^*, \cdot) \in \dreal\), such that \(\|(s,a)^\top - (s^*,a^*)^\top\| \geq \epsilon_k \) } \label{algo:rej_rule}
        \State eliminate \((s,a,s') \) from \(\dreduct\)
        \EndIf
        \EndFor
        \State \Return \( \dreduct \)
    \end{algorithmic}
\end{algorithm}

\begin{figure}[t]
    \centering
    \includegraphics[width=1\columnwidth]{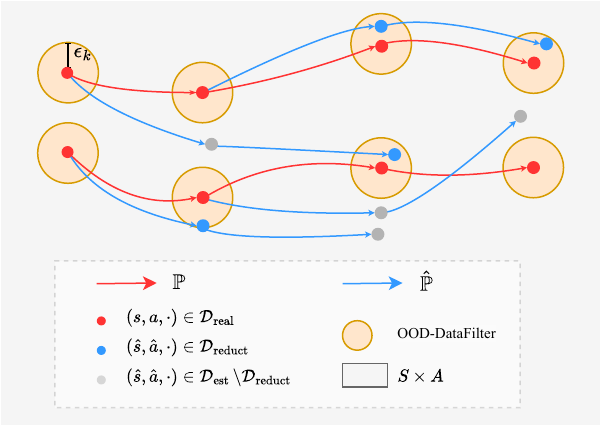}
    \caption{Visualization of out-of-distribution data filter.}
    \label{fig:OOD data filter}
\end{figure}

Line~\ref{algo:rej_rule} is the \ac{ood} rule, where we detect and discard all those state-action pairs in \( \dest \) that are far away from any state-action pairs in \( \dreal \) as shown in Fig.~\ref{fig:OOD data filter}. 
% The norm used in Line~\ref{algo:rej_rule} is the \( L_2 \) norm. 
The term \( \epsilon_k \) is the threshold of the distance between the current state-action pair and the closest state-action pair in \( \dreal \). We call \( \epsilon_k \) the \emph{reject level}.

Note that Line~\ref{algo:rej_rule} never eliminates the first-step data points from \(\dest\) if the exploration policy is deterministic, regardless of how small the reject level \(\epsilon_k\) is, if we use branched rollouts as in \citet{janner_when_2019}. This is because, in branched rollouts, the initial states are selected from \(\dreal\), meaning the closest state in \(\dreal\) is itself. 
Also, we want \(\epsilon_k\) to increase as the episode number \(k\) grows. As \(k\) increases, our estimated model becomes more accurate, allowing us to trust the data simulated from it more and permitting greater shifts, resulting in a larger \(\epsilon_k\) in our algorithm. The choice of \(\epsilon_k\) and the computation in Line~\ref{algo:rej_rule} will be discussed in Section~\ref{sec:exp}.

Note that Algorithm~\ref{algo:DataFilter} differs from an early stop strategy (smaller \(L\)). As shown in Fig.~\ref{fig:OOD data filter}, we first simulate the entire trajectory with \(L\) steps (depicted by the blue line with \(L=2\)), and then the data filter removes only the data points that are far from \(\dreal\). Notice that we reject the intermediate point in the trajectories of data points in the upper-left blue line instead of the last point, highlighting that Algorithm~\ref{algo:DataFilter} is distinct from an early stop strategy.

With \ac{ood} data filter, Dyna-style algorithms depicted in Algorithm~\ref{algo:modelbasedaddmodelfree} can be enhanced into Algorithm~\ref{algo:dyna_ood}. \ac{ood} data filter is inserted in line~\ref{algo3: OOD-data filter}. Instead of \(\dest\), the input to \texttt{ModelFreeAlgo} in line~\ref{algo3:phi update} is replaced by \(\dreduct\).
\begin{algorithm}
    \caption{\textsc{Dyna-Style learning with OOD-DataFilter}}
    \label{algo:dyna_ood}
    \begin{algorithmic}[1]
        \Require pretrained model \( \hat{\pp}^1 \), total number of episodes \( K \), rollout length \(L\),  batch size \(N\), rollout size \(M = N L\), model estimator \( \texttt{ModelEstimator} \), time horizon \( H \), model-free algorithm \texttt{ModelFreeAlgo}, OOD data filter \texttt{OOD-DataFilter}, reject level \(\epsilon_k\)
        \Ensure \( \pi_{\phi}\)
        \State \( \dreal \gets \emptyset  \), \( \dest \gets \emptyset \), \(\hp \gets \hp^1\)
        \For{\( k \in [K] \)}
        \State \(s_1 \sim \mu\)
        \For{\( h \in [H] \)}
        \State uniformly sample states from \(\dreal\), perform \(L\)-step rollout with \(\hp\) under policy \(\policy\), repeat \(N\) times until \(M \) data points are collected. Store the data in \(\dest\)   
        \State \( \dreduct \gets   \texttt{OOD-DataFilter}(\dreal, \dest, \epsilon_k) \)\label{algo3: OOD-data filter}
        \State \( \phi \gets \texttt{ModelFreeAlgo}(\phi, \dreal, \dreduct) \)\label{algo3:phi update}
        \State sample \(a_h\) from \(\policy(s_h)\) and sample \(s_{h+1}\) from \(\pp^*(s_h,a_h)\), add \( (s_h,a_h,s_{h+1})\) to \( \dreal \).
        \State \( \hpp \gets \texttt{ModelEstimator}(\hpp , \dreal) \)
        \EndFor
        \EndFor
        \State \Return \( \pi_{\phi} \)
    \end{algorithmic}
\end{algorithm}

To show the reason why Algorithm~\ref{algo:dyna_ood} leads to better acceleration, we present two results in the following subsections: 

\begin{enumerate}
    \item \emph{Bounding the data shifting}: Closer initial states lead to less shifting in trajectories simulated by the estimated model (Theorem~\ref{theorem: next_state_bound}).
    \item \emph{Bounding the value function update shifting}: Less shifting in data points of a trajectory leads to a closer update in \( Q \)-network (Theorem~\ref{theorem: bound on excess error of value function}). 
\end{enumerate}

With the above two results, we can conclude that the filtered dataset \(\dreduct\) has higher quality compared to \(\dest\), as it more closely mimics \(\dreal\) in terms of \(Q\)-network updates. 
% \sh{If closeness to $\dreal$ is the only consideration, the reviewers may ask why it is not a good idea to simply form $\dreduct$ by resampling from $\dreal$.} \yl{\( |\dreal| \) is small and not diverse enough. Also, we can increase \(\epsilon_k\) as \(k\) increases to \(\dots\)}
Thus, Algorithm~\ref{algo:DataFilter} improves the efficiency of Dyna-style algorithms, as demonstrated in Algorithm~\ref{algo:dyna_ood}. In the next two subsections, we will formally state and prove these results.

\subsection{Bounding the data shifting}

We denote \( (s, a) \) as the state-action pair in a data point taken from \( \dreal \) and \( (\hs, \ha) \) as the state-action pair in a data point taken from \( \dest \). The next state interacted with the real model is \( s' \sim \pp(s, a) =\mathcal{N} (\mu(s, a), \sigma^2(s, a ))\) and the next state simulated from the estimated model is \( \hs' \sim \hpp(\hs, \ha) = \mathcal{N} (\hat{\mu}(\hs, \ha), \hat{\sigma}^2(\hs, \ha )) \).

\begin{assumption}\label{ass:trans_lips}
    The mean of the estimated transition kernel is Lipschitz, i.e.,
    \begin{equation}
        \|
        \muhat (s,a) - \muhat (\hs, \ha)
        \|
        \leq
        L_{S \times A} \| (s,a)^\top - (\hs,\ha)^\top \|.
    \end{equation}
    for all \((s,a) \in S \times A \) and \((\hat{s},\hat{a}) \in S \times A \).
\end{assumption}
Assumption~\ref{ass:trans_lips} can be satisfied by adding a regularization term in the update of the model neural network; see~\citet{zheng_is_2022} for example.

In order to measure the model estimation error, we introduce the following definition:
\begin{definition}[Kernel density estimation (KDE) and efficient sample size]
    We use a kernel function \(K_h\) with
    bandwidth \(h\) to estimate \(\mu (s, a)\) and \(\sigma (s, a)\):

    \[ \hat{\mu} (s, a) = \frac{\sum_{i = 1}^N K_h ((s, a) - (s_i, a_i))
            s_i'}{\sum_{i = 1}^N K_h ((s, a) - (s_i, a_i))},\]

    \begin{equation}
        \hat{\sigma}^2 (s,
        a) = \frac{\sum_{i = 1}^N K_h ((s, a) - (s_i, a_i)) (s_i' - \hat{\mu} (s,
            a))^2}{\sum_{i = 1}^N K_h ((s, a) - (s_i, a_i))} .
    \end{equation}

    The next state \(s_i'\) in sample \(i\) is drawn from \(s_i' \sim \mathbb{P}
    (s_i, a_i) =\mathcal{N} (\mu (s_i, a_i), \sigma^2 (s_i, a_i))\).
    The effective sample size \(n_{\mathrm{eff}}
    (s, a)\) is defined as:

    \[ n_{\mathrm{eff}} (s, a) = \sum_{i = 1}^N K_h ((s, a) - (s_i, a_i)) . \]
\end{definition}
If we choose kernel function \( K_h \) as index function, i.e., \( K_h(x) = 1 \) if \(x = 0\) and \(0\) otherwise. The KDE estimation reduces to the maximum likelihood estimation with \( n_{\mathrm{eff}} \) being the number of visitations.

\begin{theorem}\label{theorem: next_state_bound}
    Given data point \((s, a, \cdot) \in \mathcal{D}_{\mathrm{real}}\) and \((\hs, \ha, \cdot) \in \mathcal{D}_{\mathrm{est}}\). If \( \sigma^2\) is nonzero for all state-action pairs, then, with probability \((1 - \epsilon)(1- \epsilon_{\mathrm{KDE}})\), we have
    \begin{align}
        \| \hat{s}' (\hat{s}, \hat{a}) - s' (s, a) \|
         & \leq
        L_{S\times A} \| (\hat{s},\hat{a})^\top - (s,a)^\top \|                                                                     \\
         & \quad +
        \sqrt{\frac{\sigma^2 (s,a)}{n_{\mathrm{eff}} (s, a) \epsilon_{\mathrm{KDE}}}} \\
         & \quad +
        \sqrt{\frac{\sigma^2(\hat{s},\hat{a}) +
        \hat{\sigma}^2 (s,a)}{\epsilon}},
    \end{align}
    where \(\hat{s}' (\hat{s}, \hat{a}) \sim \hat{\mathbb{P}}
    (\hat{s}, \hat{a}) \) and \(s' (s, a) \sim \mathbb{P}
    (s, a)\).
\end{theorem}

\begin{proof}
    See Appendix~\ref{pf:next_state_bound}.
\end{proof}

Theorem~\ref{theorem: next_state_bound} states that the shifting distance between a state simulated from the estimated model \(\hpp\) and a state collected by interacting with the real environment \(\pp^{*}\) is primarily controlled by three factors: the distance between the data points \((\hat{s}, \ha)\) and \((s, a)\), the variance of the real transition kernel \(\sigma^2\), and the variance of the estimated transition kernel \(\hat{\sigma}^2\). In practice, we can manually bound the variance \(\hat{\sigma}^2\) of the estimated model. The real transition model used in our experiment is deterministic with random initial states, so the terms involving \(\sigma^2\) and \(\hat{\sigma}^2\) do not dominate in our case. In other words, the state shifting can be effectively bounded by using the \ac{ood} data filter, as the dominating term is \(\| (\hat{s}, \hat{a})^\top - (s, a)^\top \|\).

In the remaining analysis, we will show that the distance between the updated parameters of \( Q \)-network using data points from \( \dest \) and \( \dreal \) is bounded by the distance between those data points.

\subsection{Bounding the value function update shifting}
To simplify the analysis, we use \ac{DQN}~\citep{mnih_human-level_2015} as the model-free algorithm instead of \ac{sac} in the theoretical analysis, thereby considering only a finite action space in this and the next subsection. The results of this analysis can also provide insights for \ac{sac}, as \ac{sac} similarly uses two \(Q\)-networks as critics.

To formally state the theorem, we denote one data point from \( \dreal \) as \(d\triangleq (s,a,s')^{\top}\) and one data point from \( \dest \) as \(\hat{d}\triangleq (\hat{s},\hat{a},\hat{s}^{'})^{\top}\).
We define the \( Q \)-network as \(Q (s, a ; \theta)\), where \(\theta\) represents the parameters.
The target value \(y\) is defined as
\begin{equation}
    y (s, a) = r (s, a) + \gamma \max_{a'} Q  (s', a' ; \theta^-),
\end{equation}
where \(s'\) is the next state, \(\theta^-\) are the parameters of the target network, which are periodically updated to match the \( Q \)-network parameters \(\theta\).
The Loss function \(L(\theta)\) of \ac{DQN} is defined as
\begin{equation}
    L (\theta) = \frac{1}{2} (y (s, a) - Q (s, a; \theta))^2.
\end{equation}
The update rule for the parameters of \( Q \)-network is
\begin{equation}
    \theta_{t + 1} \leftarrow \theta_t - \alpha \nabla_{\theta} L (\theta_t).
\end{equation}
By simple derivation, the above update rule becomes:
\begin{equation}
    \theta_{t + 1} \leftarrow \theta_t + \alpha (y (s, a) - Q (s, a ; \theta_t))
    \nabla_{\theta} Q (s, a ; \theta_t)
\end{equation}
The corresponding target and updates using a data point \((\hat{s} , \hat{a} , \hat{s} ')\) from \( \dest \) are:
\begin{equation}
    \hat{y} (\hat{s} , \hat{a} ) = r (\hat{s} , \hat{a} ) + \gamma \max_{\hat{a}'} Q
    (\hat{s} ', \hat{a} ' ; \theta^-),
\end{equation}
and
\begin{equation}
    \hat{\theta}_{t + 1} \leftarrow \theta_t + \alpha (\hat{y} (\hat{s} ,
    \hat{a} ) - Q (\hat{s} , \hat{a}  ; \theta_t)) \nabla_{\theta} Q (\hat{s} ,
    \hat{a}  ; \theta_t).
\end{equation}
Before we proceed, we make the following assumptions on the \( Q \)-network and reward function:
\begin{assumption}\label{ass: lipschiz for Q}
    The reward function, gradient function of the \(Q\)-network, and \(Q\)-network itself are Lipschitz, i.e.,
    \begin{align}
        \|r(s,a)-r(\hs,\hat{a})\|                                                                              & \leq L_1\|(s,a)^\top - (\hs,\ha)^\top\|, \\
        \|\nabla_{\theta} Q(s,a;\theta)-\nabla_{\theta} Q(\hs,\hat{a};\theta))\| & \leq L_2\|(s,a)^\top - (\hs,\ha)^\top\|, \\
        \|Q(s,a;\theta)-Q(\hs,\hat{a};\theta))\|\,                                             & \leq L_3\|(s,a)^\top - (\hs,\ha)^\top\|,
    \end{align}
    for all \((s,a) \in S \times A \) and \((\hat{s},\hat{a}) \in S \times A \).
    In addition, the $\max_{a} Q  (s, a ; \theta^-)$ is also Lipschitz with respect to $s$, i.e.,
    \begin{align}
        \|\max_{a } Q  (s, a ; \theta^-)-\max_{a } Q  (\hs, a ; \theta^-)\|\leq L_4\|s-\hs\|.
    \end{align}
\end{assumption}
The Lipschitz assumption on the reward function is determined by the environment itself. 
The assumption for the \(Q\)-network can be satisfied by regularizing its training, as demonstrated in \cite{zheng_is_2022,gouk2021regularisation}.
\begin{assumption}\label{ass: varialbes bounded}
    There are existing constants $D_1 \geq 0$, $D_2 \geq 0$ and $D_3 \geq 0$, such that, $\|r( \cdot, \cdot )\|\leq D_1$, $\|\nabla_{\theta}Q( \cdot, \cdot ; \theta )\|\leq D_2$ and $\|Q( \cdot, \cdot ; \theta )\|\leq D_3$.
\end{assumption}
The reward function of the \texttt{MuJoCo} environments we considered in the experiment section satisfies Assumption~\ref{ass: varialbes bounded}. The \(Q\)-network can also satisfy this assumption by incorporating a constraint during its training.
Now, we state the theorem:
\begin{theorem}\label{theorem: bound on excess error of value function}
    Under assumption~\ref{ass: lipschiz for Q} and~\ref{ass: varialbes bounded}, for any \(\gamma>0\) and \(\alpha>0\), it holds for any \(\theta^{-}\), \((s,a,s')\) and \((\hat{s} , \hat{a} , \hat{s} ')\):
    \begin{align}
        \| \hat{\theta}_{t + 1} - \theta_{t + 1} \| & \leq C_1(\alpha, \gamma) \| d - \hat{d} \|^2+C_2(\alpha,\gamma)\| d - \hat{d}  \|
    \end{align}
    where
    \begin{align}
        C_1(\alpha, \gamma) & =\alpha  (L_1 L_2 + L_3 L_2+\gamma L_4 L_2)                   \\
        C_2(\alpha,\gamma)  & =\alpha   L_2 D_2                                             \\
                            & \quad+ \alpha\left( \gamma L_1 + L_3 + \gamma L_4 \right) D_1 \\
                            & \quad+ \alpha(1+\gamma) L_2 D_3.
    \end{align}
\end{theorem}
\begin{proof}
    The proof is attached in Appendix~\ref{proof: bound on excess error of value function}.
\end{proof}

Combining Theorem~\ref{theorem: next_state_bound} and Theorem~\ref{theorem: bound on excess error of value function}, we derive the following proposition:
\begin{proposition}\label{coro: Q_network_distance}
    With probability \((1 - \epsilon)(1- \epsilon_{\mathrm{KDE}})\), the upper bound in theorem~\ref{theorem: bound on excess error of value function} becomes
    \begin{align}
        \| \hat{\theta}_{t + 1} - \theta_{t + 1} \| & \leq C_1(\alpha, \gamma)C_{3}^{2}\|(s,a)^\top - (\hs,\ha)^\top\|^{2} \\
                                                    & \quad + C_2(\alpha, \gamma)C_3\|(s,a)^\top - (\hs,\ha)^\top\| \\
                                                    & \quad + (\|(s,a)^\top - (\hs,\ha)^\top\|+ 1 )\mathcal{O}(\sigma_{\max}),
    \end{align}
where \(C_3 = 1+L_{S\times A}\) and \(\sigma_{\max}\) is the maximum among \(\sigma(s,a)\), \(\sigma(\hs,\ha)\), \(\hat{\sigma}(s,a)\), \(\sqrt{\hat{\sigma}(s,a)\sigma(s,a)}\) and \(\sqrt{\sigma(s,a)\sigma(\hs,\ha)}\).
\end{proposition}
\begin{proof}
    See Appendix~\ref{proof:Q_network_distance}. 
\end{proof} 
The Proposition~\ref{coro: Q_network_distance} states that when both the variances of transition kernel for the real model and estimated model are small enough, the upper bound of the shifting of \(Q\)-network update is dominated by the distance between the data point from \(\dreal\) and the data point from \(\dest\).

Considering the \ac{ood} data filter, one implication of Theorem~\ref{theorem: next_state_bound} is that data simulated from the estimated model closely approximates data interacted from the real model when \((\hs,\ha,\cdot)\in\dreduct\). As a result of that, the \( Q \)-network's parameters \( \hat{\theta}_{t + 1} \) updated by using data \(\dreduct\)  will close to the parameters \( \theta_{t + 1} \) updated by using data \( \dreal \). This increases the reliability and efficiency of the model-free training process. 

In summary, we show that \(\dreduct\) has better quality than \(\dest\) because it mimics \(\dreal\) much closer in terms of \(Q\)-network updates.
The analysis of this section provides a reason why Algorithm~\ref{algo:DataFilter} enhances the performance of Dyna-style algorithms. In the next section, we will show empirically that the \ac{mbpo} with out-of-distribution data filter outperforms \ac{mbpo} even without a model ensemble.

\section{Experiment}\label{sec:exp}

\begin{figure*}[bht]
    \centering
    \begin{subfigure}[b]{0.32\textwidth}
        \centering
        \includegraphics[width=\linewidth]{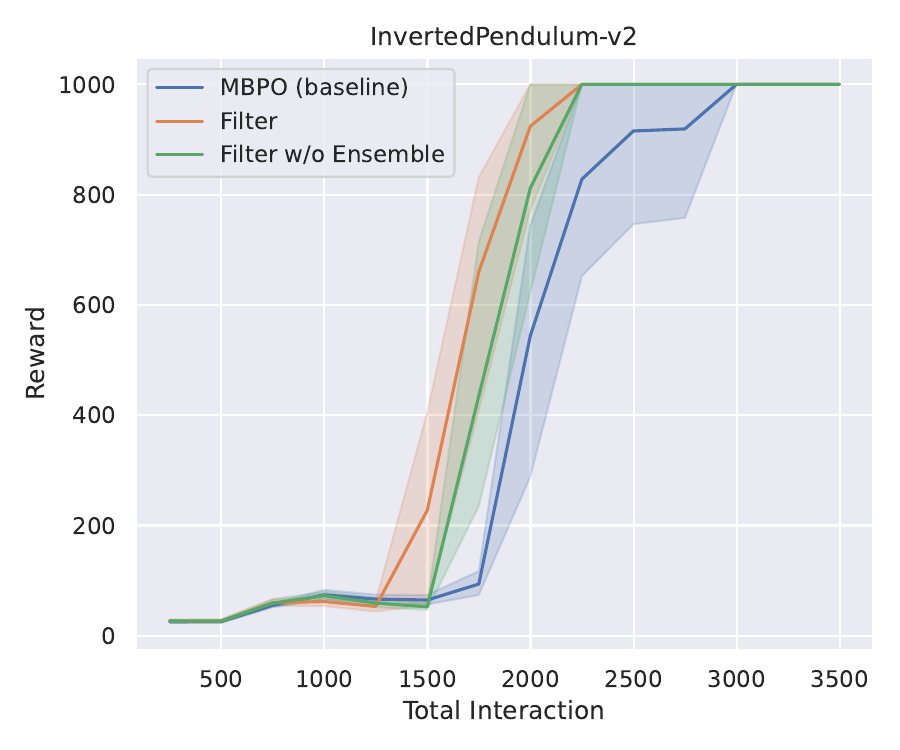}
        \caption{InvertedPendulum-V2}
        \label{fig:exp:invpend}
    \end{subfigure}
    \hfill
    \begin{subfigure}[b]{0.32\textwidth}
        \centering
        \includegraphics[width=\linewidth]{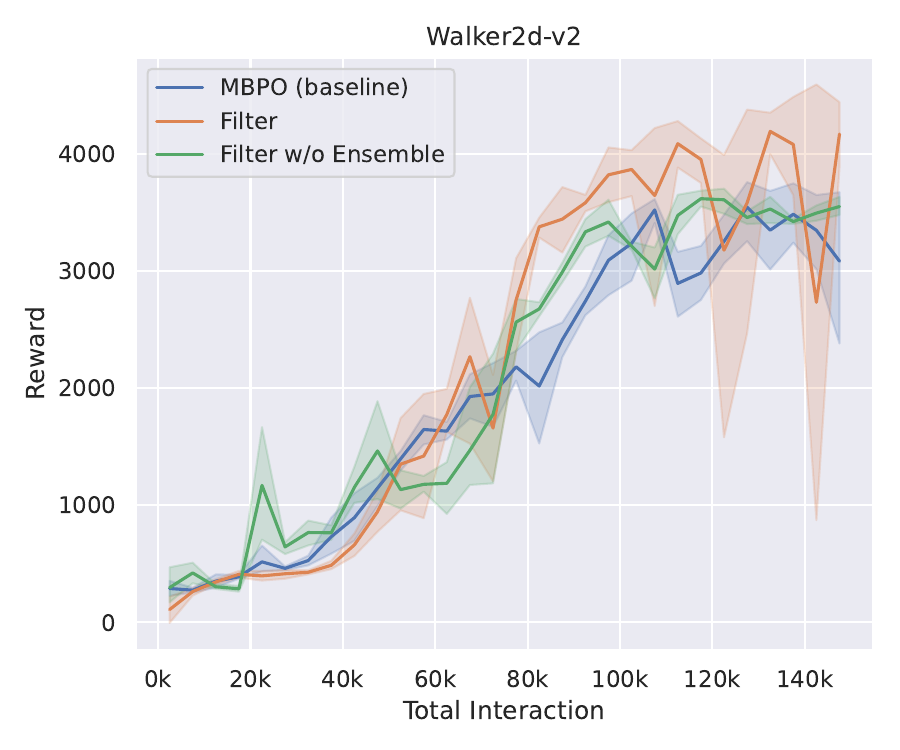}
        \caption{Walker2d-V2}
        \label{fig:exp:walker}
    \end{subfigure}
    \hfill
    \begin{subfigure}[b]{0.32\textwidth}
        \centering
        \includegraphics[width=\linewidth]{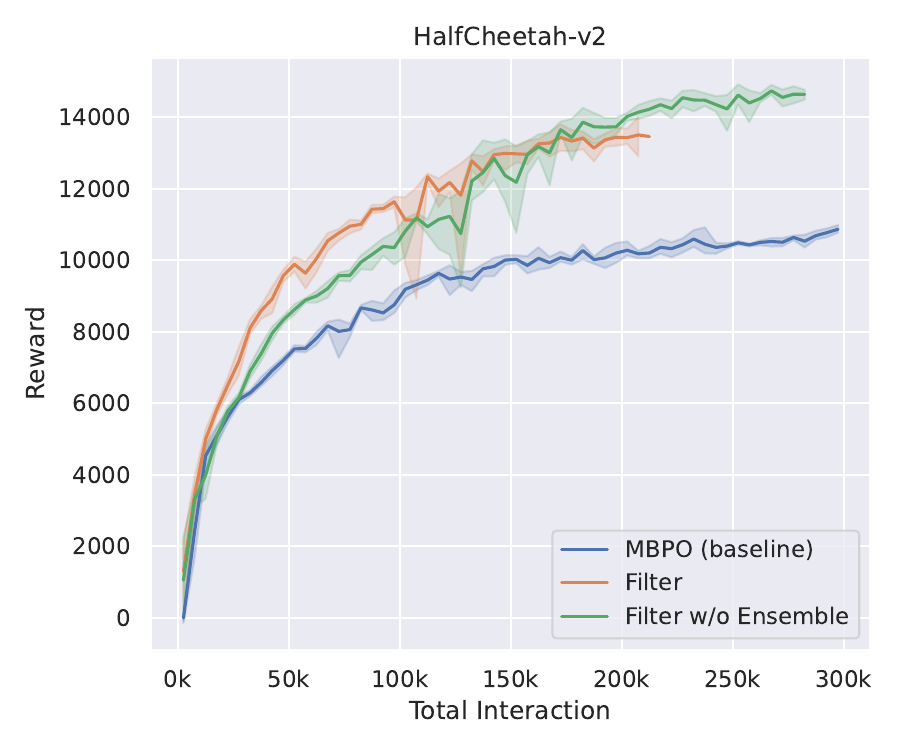}
        \caption{HalfCheetah-V2}
        \label{fig:exp:halfch}
    \end{subfigure}
    
    \caption{Experiment results for different environments in \texttt{MuJoCo}.}
    \label{fig:exp:combined}
\end{figure*}

In this section, we discuss our experimental and implementation principles, including the experiment environments, data rejection rules, important algorithm settings, and baseline algorithms, which we compare against.

\emph{Experiment Environments}: We empirically validate our theoretical results in the \texttt{MuJoCo}~\citep{todorov_mujoco_2012} engine using \texttt{Gymnasium}~\citep{brockman_openai_2016} APIs. We conduct experiments using the following environments in the \texttt{MuJoCo} engine: \texttt{InvertedPendulum-V2}, \texttt{Walker2D-V2}, and \texttt{HalfCheetah-V2}. Detailed descriptions and settings for these environments are available in \citet{todorov_mujoco_2012}. 

\subsection{Out-of-distribution data filter}

In Algorithm~\ref{algo:DataFilter}, we are applying rejections to data samples to achieve the purpose of data filter. Though in Line~\ref{algo:rej_rule} of Algorithm~\ref{algo:DataFilter}, we define the state-action pairs as the data to be taken or rejected, in practice, we find that only considering the state is sufficient. Therefore, in our experiments, we only consider data acceptance and rejection on the level of states instead of state-action pairs. For the choice of rejection level \(\epsilon_k\), we use 2 different strategies, \emph{static data rejection} and \emph{dynamic data rejection}.
In this section, we denote the estimated model in episode \(k\) as \(\hp^k\).

\emph{Static Data Rejection}: Static data rejection fits much closer to the default semantics of data rejection defined in our Algorithm~\ref{algo:DataFilter}. We pick a static value \(\epsilon\) such that if any \(\hat{s} \in \dest\) and  \(\min_{s \in \dreal} \| \hat{s} - s \| > \epsilon\), we will reject \(\hat{s}\).

\emph{Dynamic Data Rejection}: In dynamic data rejection, the rejection level changes as the training process evolves. Because the estimated model \(\hp^k\) becomes closer to the true model \(\pp^*\) as \(k\) increases, we want to choose \(\epsilon_k\) such that less data will be rejected as \(k\) increases. However, finding the optimal value \(\epsilon_k\) for each \(k\) is challenging. From a practical perspective, we first sort all the data points \(\hat{s}\) in \(\dest\) by the metric \(\min_{s \in \dreal} \| \hat{s} - s \|\) in descending order. 
Then, in episode \(1\), we choose \(\epsilon_k\) such that all model-generated data points are eliminated except the first rollout step. In episode \(K\), we accept all model-generated data. For episodes between 1 and \(K\), we use linear interpolation to determine the percentage of data points to eliminate.

In both approaches, we need to efficiently solve the minimization problem \(\min_{s \in \dreal} \| \hat{s} - s \|\). We first considered it as a nearest neighbor problem and approximate the problem as \ac{ann}. We choose one of the most robust and scalable \ac{ann} algorithms, Hierarchical Navigable Small World (HNSW), introduced by \citet{hnsw}. 
% to achieve the search operation. 
% HNSW is represented by layers of Navigable Small World, simply speaking, graphs with nodes that are high-dimension vectors connected with each other. 
%
By applying the principle of skip-lists introduced in \citet{skiplist}, HNSW achieves amortized \(O(log(n))\) time to perform search, insertion, and deletion operations.
% \sh{Empirically, how does the approximation in ANN search affect the effectiveness of the data filter compared to exact NN search?} \zd{We didn't perform compare experiments due to time limits. I will keep this comment and we can do some compare experiments after this submission in preparation for the rebuttal.} \sh{Also, does the HNSW algorithm require any preprocessing of $\dreal$? I am not very familiar with HNSW, but my understanding is that it requires building a proximity graph. If the proximity graph needs to be reconstructed every time data are added to $\dreal$, I am not sure if this overhead will outweigh the savings of ANN.}\zd{The algorithm claims the insertion time is also \(O(log(n))\). Every time a new real data is added to the graph, by using insertion, there is no need to reconstruct the whole graph. Also the ANN is actually providing an over-approximation where data filtered through this search process is guarantee to have there optimal distance to the data set less than \(\epsilon\) }

In our case, we initialize the HNSW index at the time we create the replay buffer for the data set of the real environment \(\dreal\). Whenever there is a step in the real environment, the new state \(s\) will be inserted into the HNSW index. 
%
% We designed the HNSW index to have the same capacity as the replay buffer to guarantee the fairness of comparison where no extra data point from the real environment is preserved by the HNSW index. 
%
% Whenever the index is full, a random node will be deleted, and the state will be inserted into that node's position.
During the filter process, we query the nearest neighbor from the HNSW with \(\hat{s} \in \dest\) to obtain $\min_{s \in \dreal} \| \hat{s} - s \|$.

\subsection{Tested algorithms and experiment results}

\emph{Our algorithm:} We implement Algorithm~\ref{algo:dyna_ood} using \ac{mbpo} as the Dyna-style algorithm. The parameters are nearly identical to those in \citet{janner_when_2019}, with several modifications to accommodate our data filter. First, we increase the rollout length \(L\) for simulating data from \(\hat{\pp}\). In this way, we exploit the data filter and the ensemble model by simulating longer steps from the ensemble model while filtering the \ac{ood} data.
% we decrease the rollout batch size \(N=M / L\) and 
% As a result, the total rollout size \(M\) is smaller than in \ac{mbpo} to ensure a fair comparison.
Second, besides using our data filter together with the model ensemble like in \ac{mbpo}, we also conduct experiments without the model ensemble.
The implementation and parameter details are provided in Appendix~\ref{supp:exp_details}.

\emph{Baseline algorithms:} We compare the \ac{mbpo} algorithm, using the exact parameters from \citet{janner_when_2019,liu_re_2019}. 
% We also compare our algorithm with the approach from \citet{zheng_is_2022}, a modified \ac{mbpo} algorithm with critic regularization. \yl{Although the authors claim their algorithm achieves state-of-the-art performance (and has not released the source code), our experiments show that \ac{mbpo} with critic regularization does not outperform the standard \ac{mbpo}.} 
Since other Dyna-style algorithms such as Model-Ensemble Trust-Region Policy Optimization \citep{kurutach_model-ensemble_2018} and model-free algorithms such as \ac{sac}~\citep{haarnoja_soft_2018} do not perform better than \ac{mbpo}, we do not replicate their results in this paper.

% The experimental results are shown in Fig.~\ref{fig:exp:combined}, where the vertical axis represents the reward and the horizontal axis indicates the total number of interactions with the real environment. The blue line represents the \ac{mbpo} algorithm, using the same parameters as in \cite{janner_when_2019}. The orange line represents the \ac{mbpo} algorithm with a data filter, where we decrease the batch size and increase the rollout length for deeper exploration of the estimated model. The green line is similar to the orange line but does not use a model ensemble, significantly reducing computational complexity. For each environment, we took 10 time experiment and take average as the result and plot line. The shaded area represents the variance of the 10 exxperiments.

The experimental results are shown in Fig.~\ref{fig:exp:combined}, where the vertical axis represents the cumulative reward and the horizontal axis indicates the total number of interactions with the real environment. The blue line represents the \ac{mbpo} algorithm, using the same parameters as in \cite{janner_when_2019,liu_re_2019}. The orange line represents the \ac{mbpo} algorithm with a data filter, where we increase the rollout length for a deeper exploration of the estimated model. The green line is similar to the orange line but does not use a model ensemble, significantly reducing computational complexity.
For \texttt{InvertedPendulum-v2}, we conducted 10 experiments with different random seeds, averaged the results, and plotted the lines. The shaded area represents the 95\% confidence interval among these 10 experiments. For the other two environments, we conduct one experiment for each, and the shadow region represents the variance of the testing rewards.
% \zd{this is not true, but we don't have time to do more experiments. we need to continue the experiment and put it on in the final revision.} \el{agree, we need to delete this for the first submission.}

Fig.~\ref{fig:exp:combined} shows that even without the model ensemble, the \ac{mbpo} with the out-of-distribution data filter requires fewer interactions to achieve the same suboptimality compared to \ac{mbpo}. Additionally, as shown in Fig.~\ref{fig:exp:walker} and Fig.~\ref{fig:exp:halfch}, at the end of the training, our method achieves a higher reward compared to other methods.

\section{Conclusion}
In this paper, we proposed an enhancement to Dyna-style model-based reinforcement learning algorithms through the introduction of an out-of-distribution (OOD) data filter. This approach addresses a key limitation of existing methods, where inaccurate model-generated data can hinder the efficiency and effectiveness of model-free training. By filtering out simulated data that significantly diverges from real-environment interactions, our method improves the quality of the training data, leading to faster convergence and higher final performance.

Our theoretical analysis demonstrated that the OOD data filter reduces the drift in rollout trajectories and limits the parametric shift during \(Q\)-network updates, thereby providing a more stable and reliable learning process. Empirical results on various \texttt{MuJoCo} environments validated our approach, showing that the modified \ac{mbpo} algorithm with the \ac{ood} data filter outperforms standard \ac{mbpo}, even without the need for model ensembles.

\bibliography{mbrl_aaai_2024_ref.bib, ref_added.bib}

\begin{thebibliography}{39}
\providecommand{\natexlab}[1]{#1}

\bibitem[{Ajalloeian and Stich(2021)}]{ajalloeian_convergence_2021}
Ajalloeian, A.; and Stich, S.~U. 2021.
\newblock On the {{Convergence}} of {{SGD}} with {{Biased Gradients}}.
\newblock arXiv:2008.00051.

\bibitem[{Brockman et~al.(2016)Brockman, Cheung, Pettersson, Schneider, Schulman, Tang, and Zaremba}]{brockman_openai_2016}
Brockman, G.; Cheung, V.; Pettersson, L.; Schneider, J.; Schulman, J.; Tang, J.; and Zaremba, W. 2016.
\newblock {{OpenAI Gym}}.
\newblock arXiv:1606.01540.

\bibitem[{Brunke et~al.(2021)Brunke, Greeff, Hall, Yuan, Zhou, Panerati, and Schoellig}]{brunke_safe_2021}
Brunke, L.; Greeff, M.; Hall, A.~W.; Yuan, Z.; Zhou, S.; Panerati, J.; and Schoellig, A.~P. 2021.
\newblock Safe {{Learning}} in {{Robotics}}: {{From Learning-Based Control}} to {{Safe Reinforcement Learning}}.
\newblock arXiv:2108.06266.

\bibitem[{Chua et~al.(2018)Chua, Calandra, McAllister, and Levine}]{chua_deep_2018}
Chua, K.; Calandra, R.; McAllister, R.; and Levine, S. 2018.
\newblock Deep {{Reinforcement Learning}} in a {{Handful}} of {{Trials}} Using {{Probabilistic Dynamics Models}}.
\newblock \emph{arXiv:1805.12114 [cs, stat]}.

\bibitem[{Clavera et~al.(2018)Clavera, Rothfuss, Schulman, Fujita, Asfour, and Abbeel}]{clavera_model-based_2018}
Clavera, I.; Rothfuss, J.; Schulman, J.; Fujita, Y.; Asfour, T.; and Abbeel, P. 2018.
\newblock Model-{{Based Reinforcement Learning}} via {{Meta-Policy Optimization}}.
\newblock \emph{arXiv:1809.05214 [cs, stat]}.

\bibitem[{Deisenroth and Rasmussen(2011)}]{deisenroth_pilco_2011}
Deisenroth, M.~P.; and Rasmussen, C.~E. 2011.
\newblock PILCO: a model-based and data-efficient approach to policy search.
\newblock In \emph{Proceedings of the 28th International Conference on International Conference on Machine Learning}, ICML'11, 465–472. Madison, WI, USA: Omnipress.
\newblock ISBN 9781450306195.

\bibitem[{Doll, Simon, and Daw(2012)}]{doll_ubiquity_2012}
Doll, B.~B.; Simon, D.~A.; and Daw, N.~D. 2012.
\newblock The Ubiquity of Model-Based Reinforcement Learning.
\newblock \emph{Current Opinion in Neurobiology}, 22(6): 1075--1081.

\bibitem[{Finn, Levine, and Abbeel(2016)}]{finn_guided_2016}
Finn, C.; Levine, S.; and Abbeel, P. 2016.
\newblock Guided {{Cost Learning}}: {{Deep Inverse Optimal Control}} via {{Policy Optimization}}.
\newblock arXiv:1603.00448.

\bibitem[{Gouk et~al.(2021)Gouk, Frank, Pfahringer, and Cree}]{gouk2021regularisation}
Gouk, H.; Frank, E.; Pfahringer, B.; and Cree, M.~J. 2021.
\newblock Regularisation of neural networks by enforcing lipschitz continuity.
\newblock \emph{Machine Learning}, 110: 393--416.

\bibitem[{Haarnoja et~al.(2018)Haarnoja, Zhou, Abbeel, and Levine}]{haarnoja_soft_2018}
Haarnoja, T.; Zhou, A.; Abbeel, P.; and Levine, S. 2018.
\newblock Soft {{Actor-Critic}}: {{Off-Policy Maximum Entropy Deep Reinforcement Learning}} with a {{Stochastic Actor}}.
\newblock arXiv:1801.01290.

\bibitem[{Hsu et~al.(2020)Hsu, Shen, Jin, and Kira}]{hsu_generalized_2020}
Hsu, Y.-C.; Shen, Y.; Jin, H.; and Kira, Z. 2020.
\newblock Generalized {{ODIN}}: {{Detecting Out-of-Distribution Image Without Learning From Out-of-Distribution Data}}.
\newblock In \emph{2020 {{IEEE}}/{{CVF Conference}} on {{Computer Vision}} and {{Pattern Recognition}} ({{CVPR}})}, 10948--10957. Seattle, WA, USA: IEEE.
\newblock ISBN 978-1-72817-168-5.

\bibitem[{Janner et~al.(2019)Janner, Fu, Zhang, and Levine}]{janner_when_2019}
Janner, M.; Fu, J.; Zhang, M.; and Levine, S. 2019.
\newblock When to {{Trust Your Model}}: {{Model-Based Policy Optimization}}.
\newblock https://arxiv.org/abs/1906.08253v3.

\bibitem[{Johannink et~al.(2019)Johannink, Bahl, Nair, Luo, Kumar, Loskyll, Ojea, Solowjow, and Levine}]{johannink_residual_2019}
Johannink, T.; Bahl, S.; Nair, A.; Luo, J.; Kumar, A.; Loskyll, M.; Ojea, J.~A.; Solowjow, E.; and Levine, S. 2019.
\newblock Residual {{Reinforcement Learning}} for {{Robot Control}}.
\newblock In \emph{2019 {{International Conference}} on {{Robotics}} and {{Automation}} ({{ICRA}})}, 6023--6029.

\bibitem[{Kober and Peters(2014)}]{kober_reinforcement_2014}
Kober, J.; and Peters, J. 2014.
\newblock Reinforcement {{Learning}} in {{Robotics}}: {{A Survey}}.
\newblock In Kober, J.; and Peters, J., eds., \emph{Learning {{Motor Skills}}: {{From Algorithms}} to {{Robot Experiments}}}, 9--67. Cham: Springer International Publishing.
\newblock ISBN 978-3-319-03194-1.

\bibitem[{Kurutach et~al.(2018)Kurutach, Clavera, Duan, Tamar, and Abbeel}]{kurutach_model-ensemble_2018}
Kurutach, T.; Clavera, I.; Duan, Y.; Tamar, A.; and Abbeel, P. 2018.
\newblock Model-{{Ensemble Trust-Region Policy Optimization}}.
\newblock \emph{arXiv:1802.10592 [cs]}.

\bibitem[{Lai et~al.(2022)Lai, Shen, Zhang, Huang, Zhang, Tang, Yu, and Li}]{lai_effective_2022}
Lai, H.; Shen, J.; Zhang, W.; Huang, Y.; Zhang, X.; Tang, R.; Yu, Y.; and Li, Z. 2022.
\newblock On {{Effective Scheduling}} of {{Model-based Reinforcement Learning}}.
\newblock arXiv:2111.08550.

\bibitem[{Levine and Abbeel(2014)}]{levine_learning_2014}
Levine, S.; and Abbeel, P. 2014.
\newblock Learning {{Neural Network Policies}} with {{Guided Policy Search}} under {{Unknown Dynamics}}.
\newblock In \emph{Advances in {{Neural Information Processing Systems}}}, volume~27. Curran Associates, Inc.

\bibitem[{Levine, Wagener, and Abbeel(2015)}]{levine_learning_2015}
Levine, S.; Wagener, N.; and Abbeel, P. 2015.
\newblock Learning {{Contact-Rich Manipulation Skills}} with {{Guided Policy Search}}.
\newblock arXiv:1501.05611.

\bibitem[{Li(2018)}]{li_deep_2018}
Li, Y. 2018.
\newblock Deep {{Reinforcement Learning}}: {{An Overview}}.
\newblock arXiv:1701.07274.

\bibitem[{Li and Han(2022)}]{li_accelerating_2022}
Li, Y.; and Han, S. 2022.
\newblock Accelerating {{Model-Free Policy Optimization Using Model-Based Gradient}}: {{A Composite Optimization Perspective}}.
\newblock In \emph{Proceedings of {{The}} 4th {{Annual Learning}} for {{Dynamics}} and {{Control Conference}}}, 304--315. PMLR.

\bibitem[{Liu et~al.(2023)Liu, Shen, He, Zhang, Xu, Yu, and Cui}]{liu_towards_2023}
Liu, J.; Shen, Z.; He, Y.; Zhang, X.; Xu, R.; Yu, H.; and Cui, P. 2023.
\newblock Towards {{Out-Of-Distribution Generalization}}: {{A Survey}}.
\newblock arXiv:2108.13624.

\bibitem[{Liu, Xu, and Pan(2019)}]{liu_re_2019}
Liu, Y.; Xu, J.; and Pan, Y. 2019.
\newblock [{{Re}}] {{When}} to {{Trust Your Model}}: {{Model-Based PolicyOptimization}}.
\newblock In \emph{2019 {{Conference}} on {{Neural Information Processing Systems}}}.

\bibitem[{Luo et~al.(2021)Luo, Xu, Li, Tian, Darrell, and Ma}]{luo_algorithmic_2021}
Luo, Y.; Xu, H.; Li, Y.; Tian, Y.; Darrell, T.; and Ma, T. 2021.
\newblock Algorithmic {{Framework}} for {{Model-based Deep Reinforcement Learning}} with {{Theoretical Guarantees}}.
\newblock arXiv:1807.03858.

\bibitem[{Malkov and Yashunin(2018)}]{hnsw}
Malkov, Y.~A.; and Yashunin, D.~A. 2018.
\newblock Efficient and robust approximate nearest neighbor search using hierarchical navigable small world graphs.
\newblock \emph{IEEE transactions on pattern analysis and machine intelligence}, 42(4): 824--836.

\bibitem[{Mnih et~al.(2015)Mnih, Kavukcuoglu, Silver, Rusu, Veness, Bellemare, Graves, Riedmiller, Fidjeland, Ostrovski, Petersen, Beattie, Sadik, Antonoglou, King, Kumaran, Wierstra, Legg, and Hassabis}]{mnih_human-level_2015}
Mnih, V.; Kavukcuoglu, K.; Silver, D.; Rusu, A.~A.; Veness, J.; Bellemare, M.~G.; Graves, A.; Riedmiller, M.; Fidjeland, A.~K.; Ostrovski, G.; Petersen, S.; Beattie, C.; Sadik, A.; Antonoglou, I.; King, H.; Kumaran, D.; Wierstra, D.; Legg, S.; and Hassabis, D. 2015.
\newblock Human-Level Control through Deep Reinforcement Learning.
\newblock \emph{Nature}, 518(7540): 529--533.

\bibitem[{Nasvytis et~al.(2024)Nasvytis, Sandbrink, Foerster, Franzmeyer, and {de Witt}}]{nasvytis_rethinking_2024}
Nasvytis, L.; Sandbrink, K.; Foerster, J.; Franzmeyer, T.; and {de Witt}, C.~S. 2024.
\newblock Rethinking {{Out-of-Distribution Detection}} for {{Reinforcement Learning}}: {{Advancing Methods}} for {{Evaluation}} and {{Detection}}.
\newblock arXiv:2404.07099.

\bibitem[{Peters, Vijayakumar, and Schaal(Sept.29-30 2003)}]{peters_reinforcement_2003}
Peters, J.; Vijayakumar, S.; and Schaal, S. Sept.29-30 2003.
\newblock Reinforcement {{Learning}} for {{Humanoid Robotics}}.
\newblock In \emph{Third {{IEEE-RAS International Conference}} on {{Humanoid Robots}}, {{Karlsruhe}}, {{Germany}}, {{Sept}}.29-30}.

\bibitem[{Pugh(1990)}]{skiplist}
Pugh, W. 1990.
\newblock Skip lists: a probabilistic alternative to balanced trees.
\newblock \emph{Commun. ACM}, 33(6): 668–676.

\bibitem[{Sedlmeier et~al.(2019)Sedlmeier, Gabor, Phan, Belzner, and {Linnhoff-Popien}}]{sedlmeier_uncertainty-based_2019}
Sedlmeier, A.; Gabor, T.; Phan, T.; Belzner, L.; and {Linnhoff-Popien}, C. 2019.
\newblock Uncertainty-{{Based Out-of-Distribution Detection}} in {{Deep Reinforcement Learning}}.
\newblock arXiv:1901.02219.

\bibitem[{Shen et~al.(2023)Shen, Lai, Liu, Zhao, Yu, and Zhang}]{shen_adaptation_2023}
Shen, J.; Lai, H.; Liu, M.; Zhao, H.; Yu, Y.; and Zhang, W. 2023.
\newblock Adaptation {{Augmented Model-based Policy Optimization}}.
\newblock \emph{Journal of Machine Learning Research}, 24(218): 1--35.

\bibitem[{Sutton(1991)}]{sutton_dyna_1991}
Sutton, R.~S. 1991.
\newblock Dyna, an Integrated Architecture for Learning, Planning, and Reacting.
\newblock \emph{SIGART Bull.}, 2(4): 160--163.

\bibitem[{Sutton and Barto(2018)}]{sutton_reinforcement_2018}
Sutton, R.~S.; and Barto, A.~G. 2018.
\newblock \emph{Reinforcement {{Learning}}: {{An Introduction}}}.
\newblock The MIT Press.
\newblock ISBN 978-0-262-35270-3.

\bibitem[{Tassa, Erez, and Todorov(2012)}]{tassa_synthesis_2012-1}
Tassa, Y.; Erez, T.; and Todorov, E. 2012.
\newblock Synthesis and Stabilization of Complex Behaviors through Online Trajectory Optimization.
\newblock In \emph{2012 {{IEEE}}/{{RSJ International Conference}} on {{Intelligent Robots}} and {{Systems}}}, 4906--4913.

\bibitem[{Todorov, Erez, and Tassa(2012)}]{todorov_mujoco_2012}
Todorov, E.; Erez, T.; and Tassa, Y. 2012.
\newblock {{MuJoCo}}: {{A}} Physics Engine for Model-Based Control.
\newblock In \emph{2012 {{IEEE}}/{{RSJ International Conference}} on {{Intelligent Robots}} and {{Systems}}}, 5026--5033.

\bibitem[{Wang and Giannakis(2020)}]{wang_finite-time_2020}
Wang, G.; and Giannakis, G.~B. 2020.
\newblock Finite-{{Time Error Bounds}} for {{Biased Stochastic Approximation}} with {{Applications}} to {{Q-Learning}}.
\newblock In \emph{Proceedings of the {{Twenty Third International Conference}} on {{Artificial Intelligence}} and {{Statistics}}}, 3015--3024. PMLR.

\bibitem[{Wang et~al.(2019)Wang, Bao, Clavera, Hoang, Wen, Langlois, Zhang, Zhang, Abbeel, and Ba}]{wang_benchmarking_2019}
Wang, T.; Bao, X.; Clavera, I.; Hoang, J.; Wen, Y.; Langlois, E.; Zhang, S.; Zhang, G.; Abbeel, P.; and Ba, J. 2019.
\newblock Benchmarking {{Model-Based Reinforcement Learning}}.
\newblock arXiv:1907.02057.

\bibitem[{Yang et~al.(2024)Yang, Zhou, Li, and Liu}]{yang_generalized_2024}
Yang, J.; Zhou, K.; Li, Y.; and Liu, Z. 2024.
\newblock Generalized {{Out-of-Distribution Detection}}: {{A Survey}}.
\newblock arXiv:2110.11334.

\bibitem[{Zhang et~al.(2016)Zhang, Kahn, Levine, and Abbeel}]{zhang_learning_2016}
Zhang, T.; Kahn, G.; Levine, S.; and Abbeel, P. 2016.
\newblock Learning {{Deep Control Policies}} for {{Autonomous Aerial Vehicles}} with {{MPC-Guided Policy Search}}.
\newblock arXiv:1509.06791.

\bibitem[{Zheng et~al.(2022)Zheng, Wang, Xu, and Huang}]{zheng_is_2022}
Zheng, R.; Wang, X.; Xu, H.; and Huang, F. 2022.
\newblock Is {{Model Ensemble Necessary}}? {{Model-based RL}} via a {{Single Model}} with {{Lipschitz Regularized Value Function}}.
\newblock In \emph{The {{Eleventh International Conference}} on {{Learning Representations}}}.

\end{thebibliography}
\newpage
\section*{AAAI25 Reproducibility Checklist}

This paper:
\begin{itemize}
    \item Includes a conceptual outline and/or pseudocode description of AI methods introduced. yes
    \item Clearly delineates statements that are opinions, hypothesis, and speculation from objective facts and results. yes
    \item Provides well marked pedagogical references for less-familiare readers to gain background necessary to replicate the paper. yes
    
\end{itemize}
Does this paper make theoretical contributions? yes
\begin{itemize}
    \item All assumptions and restrictions are stated clearly and formally. yes
    \item All novel claims are stated formally (e.g., in theorem statements). yes
    \item Proofs of all novel claims are included. yes
    \item Proof sketches or intuitions are given for complex and/or novel results. yes
    \item Appropriate citations to theoretical tools used are given. yes
    \item All theoretical claims are demonstrated empirically to hold. yes
    \item All experimental code used to eliminate or disprove claims is included. yes
    
\end{itemize}
Does this paper rely on one or more datasets? no

Does this paper include computational experiments? yes

\begin{itemize}
    \item Any code required for pre-processing data is included in the appendix. yes
    \item All source code required for conducting and analyzing the experiments is included in a code appendix. yes
    \item All source code required for conducting and analyzing the experiments will be made publicly available upon publication of the paper with a license that allows free usage for research purposes. yes
    \item All source code implementing new methods have comments detailing the implementation, with references to the paper where each step comes from. yes
    \item If an algorithm depends on randomness, then the method used for setting seeds is described in a way sufficient to allow replication of results. yes
    \item This paper specifies the computing infrastructure used for running experiments (hardware and software), including GPU/CPU models; amount of memory; operating system; names and versions of relevant software libraries and frameworks. yes
    \item This paper formally describes evaluation metrics used and explains the motivation for choosing these metrics. yes
    \item Analysis of experiments goes beyond single-dimensional summaries of performance (e.g., average; median) to include measures of variation, confidence, or other distributional information. yes
    \item The significance of any improvement or decrease in performance is judged using appropriate statistical tests (e.g., Wilcoxon signed-rank). yes
    \item This paper lists all final (hyper-)parameters used for each model/algorithm in the paper’s experiments. yes
    \item This paper states the number and range of values tried per (hyper-) parameter during development of the paper, along with the criterion used for selecting the final parameter setting. yes

\end{itemize}

\newpage
\appendix
\onecolumn

\appendix
\section{Lemmas}\label{proof: upper bound for joint normal}

\begin{proposition} \label{prpo:next_state_bound}
    For any random variable \(s\sim\mathcal{N} (\mu,\sigma^{2}) \) and \(\hat{s}\sim\mathcal{N} (\hat{\mu},\hat{\sigma}^{2}) \), while \(s\) and \(\hat{s}\) are independent, then with probability \(1 - \epsilon\),
    \[ \| s - \hat{s} \| \leq \sqrt{\frac{\sigma^2 + \hat{\sigma}^2}{\epsilon}}
        + \| \mu - \hat{\mu} \| . \]
\end{proposition}
\begin{proof}
    To make the mean zero, we can consider $\delta s^{\prime} = \delta s - (\mu
        - \hat{\mu})$. This will give us: $\delta s^{\prime} \sim N (0, \sigma^2 +
        \hat{\sigma}^2)$. Now, using the Chebyshev inequality, for any $\epsilon >
        0$:
    \[ P (| \delta s^{\prime} | \geq k) \leq \frac{\text{Var} (\delta
            s^{\prime})}{k^2} . \]
    Substituting the variance:
    \[ P (| \delta s^{\prime} | \geq k) \leq \frac{\sigma^2 +
            \hat{\sigma}^2}{k^2} . \]
    To find $k$ such that the probability is $1 - \epsilon$:
    \[ \frac{\sigma^2 + \hat{\sigma}^2}{k^2} = \epsilon \Rightarrow k =
        \sqrt{\frac{\sigma^2 + \hat{\sigma}^2}{\epsilon}} . \]
    Therefore, with probability $1 - \epsilon$:
    \[ | \delta s^{\prime} | \leq \sqrt{\frac{\sigma^2 +
                \hat{\sigma}^2}{\epsilon}} . \]
    Since $\delta s = \delta s^{\prime} + (\mu - \hat{\mu})$, we need to add $\|
        \mu - \hat{\mu} \|$ to the bound:
    \[ | \delta s |\leq \sqrt{\frac{\sigma^2 + \hat{\sigma}^2}{\epsilon}}
        + \| \mu - \hat{\mu} \| . \]
    Thus, with probability $1 - \epsilon$:
    \[ |s - \hat{s} |\leq \sqrt{\frac{\sigma^2 +
                \hat{\sigma}^2}{\epsilon}} + \| \mu - \hat{\mu} \| . \]

\end{proof}

\section{Proof of Theorem~\ref{theorem: next_state_bound}}\label{pf:next_state_bound}
\begin{proof}
    By Proposition~\ref{prpo:next_state_bound}, the following inequality holds for all \( (s,a,\cdot) \in \dreal \) and \(  (\hs,\hat{a},\cdot) \in \dest \)
    \begin{align}
        \| \hat{s}' (\hat{s}, \ha) - s' (s, a) \|
         & \leq
        \sqrt{\frac{\sigma^2 (\hat{s},\ha) +
        \hat{\sigma}^2 (s,a)}{\epsilon}} \\
         & \quad+
        \| \muhat (\hat{s},\ha) - \mu (s,a)
        \|
    \end{align}
    with right term being bounded by:
    \begin{align}
        \|
        \muhat (\hat{s},\ha) - \mu(s,a)
        \|
         & \leq
        \|
        \muhat (\hat{s},\ha) - \muhat (s,a)
        \|         \\
         & \quad +
        \|
        \muhat (s,a) - \mu  (s,a)
        \|
    \end{align}
    The first term of the right side can be bounded by Assumption~\ref{ass:trans_lips}:
    \begin{align}
        \|
        \muhat(\hat{s},\ha) - \muhat (s,a)
        \|
         & \leq  L_{S \times A} \| (\hs,\ha)^\top -(s,a)^\top\|.
    \end{align}
    The second term \( \| \mu (s,a) - \hat{\mu} (s,a) \| \) stands for estimation error.
    Using Chebyshev's inequality, we get the following bounds with probability
    \(1 - \epsilon_{\mathrm{KDE}}\):

    \[
        (\mu (s,a) - \hat{\mu} (s,a) )^2 \leq \frac{\sigma^2(s,a)}{n_{\mathrm{eff}} (s, a) \epsilon_{\mathrm{KDE}}} .
    \]
    %%%%%%%%%%%%%%%
    Combining all these upper bounds gives us:
    \begin{align}
        \| \hat{s}' (\hat{s}, \hat{a}) - s' (s, a) \|
         & \leq
        L_{S\times A} \| (\hat{s},\hat{a})^\top - (s,a)^\top \|                                                                     \\
         & \quad +
        \sqrt{\frac{\sigma^2 (s,a)}{n_{\mathrm{eff}} (s, a) \epsilon_{\mathrm{KDE}}}} \\
         & \quad +
        \sqrt{\frac{\sigma^2(\hat{s},\hat{a}) +
        \hat{\sigma}^2 (s,a)}{\epsilon}}.
    \end{align}
\end{proof}

\section{Proof of Theorem~\ref{theorem: bound on excess error of value function}}
\begin{proof}\label{proof: bound on excess error of value function}
    When the parameters of \(Q\)-Network get updated under real environment and estimated model, the difference is:

    \begin{align}
        \| \hat{\theta}_{t + 1} - \theta_{t + 1} \| & = \left\| \theta_t + \alpha
        (\hat{y} (\hat{s} , \hat{a} ) - Q (\hat{s}, \hat{a} ; \theta_t))
        \nabla_{\theta} Q (\hat{s}, \hat{a} ; \theta_t) - \theta_t - \alpha (y
        (s, a) - Q (s, a; \theta_t)) \nabla_{\theta} Q (s, a ; \theta_t) \right\|                                        \\
                                                    & = \left\| \alpha (\hat{y} (\hat{s}, \hat{a}) - Q (\hat{s}, \hat{a}
        ; \theta_t)) \nabla_{\theta} Q (\hat{s}, \hat{a} ; \theta_t) - \alpha (y
        (s, a) - Q (s, a; \theta_t)) \nabla_{\theta} Q (s, a ; \theta_t) \right\|
    \end{align}
    After plugging in the corresponding target value functions, it becomes:
    \begin{align}
        \| \hat{\theta}_{t + 1} - \theta_{t + 1} \| & =\alpha \left\|T_{1}(s,a,\hat{s},\hat{a};\theta)+ T_{2}(s,a,\hat{s},\hat{a};\theta)+T_{3}(s,a,\hat{s},\hat{a};\theta)\right\| \\
                                                    & \leq\alpha \left\|T_1 \right\|+\left\|T_2 \right\|+\left\|T_3 \right\|
    \end{align}
    where
    \begin{align}
        T_{1}(s,a,\hat{s},\hat{a};\theta) & = r (\hat{s}, \hat{a})
        \nabla_{\theta} Q (\hat{s}, \hat{a} ; \theta_t) - r (s, a)
        \nabla_{\theta} Q (s, a ; \theta_t),                                                                                                                                             \\
        T_{2}(s,a,\hat{s},\hat{a};\theta) & = \gamma
        \max_{a'} Q  (\hat{s}', \hat{a}' ; \theta^-) \nabla_{\theta} Q
        (\hat{s}, \hat{a} ; \theta_t) - \gamma \max_{a'} Q  (s', a' ; \theta^-)
        \nabla_{\theta} Q (s, a ; \theta_t),                                                                                                                                             \\
        T_{3}(s,a,\hat{s},\hat{a};\theta) & =  Q (s, a ; \theta_t) \nabla_{\theta} Q (s, a ; \theta_t)  -Q (\hat{s},\hat{a} ; \theta_t) \nabla_{\theta} Q (\hat{s}, \hat{a} ; \theta_t).
    \end{align}
    It's obvious that \(\|(s,a)^{\top}-(\hs,\ha)^{\top}\|\leq\|d-\hat{d}\|\). Then under assumption~\ref{ass: lipschiz for Q}, we have
    \begin{align}
        \left\|T_1(s,a,\hat{s},\hat{a};\theta)\right\| & =\left( r (\hat{s}_1, \hat{a}_1) \nabla_{\theta} Q (\hat{s}_1, \hat{a}_1 ;\theta_t) - r (\hat{s}_1, \hat{a}_1) \nabla_{\theta} Q (s, a ; \theta_t) + r(\hat{s}_1, \hat{a}_1) \nabla_{\theta} Q (s, a ; \theta_t) - r (s, a)\nabla_{\theta} Q (s, a ; \theta_t) \right) \\
                                                       & \leq \| r (\hat{s}_1, \hat{a}_1) \| L_2 \|d-\hat{d}\| + L_1 \|d-\hat{d}\| \left\|\nabla_{\theta} Q (s, a ; \theta_t) \right\|
        \\
        \left\|T_2(s,a,\hat{s},\hat{a};\theta)\right\| & \leq\left\| \max_{a_1'} Q  (\hat{s}_1', \hat{a}_1' ; \theta^-) \nabla_{\theta} Q(\hat{s}_1, \hat{a}_1 ; \theta_t) - \max_{a_1'} Q  (\hat{s}_1',\hat{a}_1' ;\theta^-) \nabla_{\theta} Q (s, a ; \theta_t)\right\|                                                       \\
                                                       & \quad+\left\| \max_{a_1'} Q (\hat{s}_1',\hat{a}_1' ; \theta^-)\nabla_{\theta} Q (s, a ; \theta_t) - \max_{a'} Q  (s', a' ;\theta^-) \nabla_{\theta} Q (s, a ; \theta_t)\right\|                                                                                        \\
                                                       & \leq \left\| \max_{a_1'} Q  (\hat{s}_1', \hat{a}_1' ; \theta^-)  \right\| L_2\|d-\hat{d}\| + \left\| \nabla_{\theta} Q (s, a ;\theta_t) \right\| L_4 \|s-\hat{s}\|                                                                                                     \\
        \left\|T_3(s,a,\hat{s},\hat{a};\theta)\right\| & \leq\left\| Q (\hat{s}_1, \hat{a}_1 ; \theta_t)\nabla_{\theta} Q (\hat{s}_1, \hat{a}_1 ;\theta_t) - Q (s, a ; \theta_t)  \nabla_{\theta} Q (\hat{s}_1, \hat{a}_1 ;\theta_t) \right\|                                                                                   \\
                                                       & \quad+ \left\|Q (s, a ;\theta_t)  \nabla_{\theta} Q (\hat{s}_1, \hat{a}_1 ;\theta_t) - Q (s, a ; \theta_t) \nabla_{\theta} Q (s, a ;\theta_t)\right\|                                                                                                                  \\
                                                       & \leq L_3 \|d-\hat{d}\| \left\| \nabla_{\theta} Q(\hat{s}_1, \hat{a}_1 ; \theta_t)\right\| + L_2 \|d-\hat{d}\| \| Q (s, a ; \theta_t) \|
    \end{align}
    Let's write $\delta d=\|d-\hat{d}\|$ and $\delta s=\|s-\hat{s}\|$ in short. Combing all three inequalities, together with assumption~\ref{ass: lipschiz for Q} we have
    \begin{align}
        \| \hat{\theta}_{t + 1} - \theta_{t + 1} \| & \leq \alpha \| r(\hat{s}_1, \hat{a}_1) \| L_2 \delta d + L_1 \delta d
        \left\| \nabla_{\theta} Q (s, a ; \theta_t) \right\|                                                                                                                                                                               \\
                                                    & \quad+ \alpha\gamma \left(\left\| \max_{a_1'} Q  (\hat{s}_1', \hat{a}_1'; \theta^-)  \right\| L_2 \delta d + \left\|\nabla_{\theta} Q (s, a ;\theta_t) \right\| L_4 \delta s \right) \\
                                                    & \quad+\alpha L_3 \delta d \left\|\nabla_{\theta} Q (\hat{s}_1, \hat{a}_1 ;\theta_t) \right\| + L_2 \delta d \|Q (s, a ; \theta_t) \|                                                 \\
                                                    & \leq \alpha (L_1 \delta d + r (s, a)) L_2 \delta d + \alpha \gamma
        (L_1 \delta d + L_4 \delta s) \left\|
        \nabla_{\theta} Q (s, a ; \theta_t) \right\|                                                                                                                                                                                       \\
                                                    & \quad+ \alpha \gamma ( L_4 \delta
        s + \|  \max_{a } Q  (s, a ; \theta^-) \|) L_2 \delta d + \alpha L_3
        \delta d \left( L_2 \delta d + \left\| \nabla_{\theta} Q (s, a ; \theta_t)
        \right\| \right)                                                                                                                                                                                                                   \\
                                                    & \quad+ \alpha L_2 \delta d \| Q (s, a ; \theta_t)\|                                                                                                                                  \\
                                                    & \leq \alpha (L_1 L_2 + L_3 L_2 + \gamma L_4 L_2) {(\delta d)}^{2}  + \alpha r (s, a)
        L_2 \delta d                                                                                                                                                                                                                       \\
                                                    & \quad+\alpha (\gamma L_1 + L_3 + \gamma L_4)  \left\|
        \nabla_{\theta} Q (s, a ; \theta_t) \right\|\delta d + \alpha \gamma \| \max_{a } Q
        (s, a ; \theta^-) \| L_2 \delta d + \alpha L_2  \| Q (s, a ; \theta_t)
        \|\delta d
    \end{align}
    Adding assumption~\ref{ass: varialbes bounded} to the upper bound then we derive:
    \begin{align}
        \| \hat{\theta}_{t + 1} - \theta_{t + 1} \| & \leq C_1(\alpha, \gamma) \| d - \hat{d} \|^2+C_2(\alpha,\gamma)\| d - \hat{d}  \|
    \end{align}
    where
    \begin{align}
        C_1(\alpha, \gamma) & =\alpha  (L_1 L_2 + L_3 L_2+\gamma L_4 L_2)                   \\
        C_2(\alpha,\gamma)  & =\alpha   L_2 D_2                                             \\
                            & \quad+ \alpha\left( \gamma L_1 + L_3 + \gamma L_4 \right) D_1 \\
                            & \quad+ \alpha(1+\gamma) L_2 D_3.
    \end{align}
\end{proof}

\section{Proof of Proposition~\ref{coro: Q_network_distance}}\label{proof:Q_network_distance}
\begin{proof}
    From Theorem~\ref{theorem: bound on excess error of value function}, with probability \((1-\epsilon)(1-\epsilon_{\mathrm{KDE}})\) we have 
\begin{align}
    \begin{array}{ll}
  \|d - \hat{d} \| & \leq \|(\hat{s}, \hat{a})^{\top} - (s, a)^{\top} \| + \|
  \hat{s}' - s' \|\\
  & \leq \|(\hat{s}, \hat{a})^{\top} - (s, a)^{\top} \| + L_{S \times A}
  \|(\hat{s}, \hat{a})^{\top} - (s, a)^{\top} \| + \sqrt{\frac{\sigma^2 (s,
  a)}{n_{\mathrm{eff}} (s, a) \epsilon_{\mathrm{KDE}}}} + \sqrt{\frac{\sigma^2
  (\hat{s}, \hat{a})+\hat{\sigma}^2 (s,a)}{\epsilon}}\\
  & \leq (1 + L_{S \times A})  \|(\hat{s}, \hat{a})^{\top} - (s, a)^{\top} \|
  + \sqrt{\frac{\sigma^2 (s, a)}{n_{\mathrm{eff}} (s, a)
  \epsilon_{\mathrm{KDE}}}} + \sqrt{\frac{\sigma^2 (\hat{s},\hat{a})+\hat{\sigma}^2 (s,a)}{\epsilon}}
\end{array}
\end{align}
Taking square on both side, we get
\begin{align}
    \begin{array}{ll}
  \|d - \hat{d} \|^2 & \leq \left( (1 + L_{S \times A}) \|(\hat{s},
  \hat{a})^{\top} - (s, a)^{\top} \|+ \sqrt{\frac{\sigma^2 (s,
  a)}{n_{\mathrm{eff}} (s, a) \epsilon_{\mathrm{KDE}}}} + \sqrt{\frac{\sigma^2
  (\hat{s}, \hat{a})+\hat{\sigma}^2 (s,a)}{\epsilon}} \right)^2\\
  & \leq ((1 + L_{S \times A}) \|(\hat{s}, \hat{a})^{\top} - (s, a)^{\top}
  \|)^2 + 2 (1 + L_{S \times A})  \|(\hat{s}, \hat{a})^{\top} - (s, a)^{\top}
  \| \sqrt{\frac{\sigma^2 (s, a)}{n_{\mathrm{eff}} (s, a)
  \epsilon_{\mathrm{KDE}}}}\\
  & \quad + \frac{\sigma^2 (s, a)}{n_{\mathrm{eff}} (s, a)
  \epsilon_{\mathrm{KDE}}} + 2 (1 + L_{S \times A})  \|(\hat{s},
  \hat{a})^{\top} - (s, a)^{\top} \| \sqrt{\frac{\sigma^2 (\hat{s},
  \hat{a})+\hat{\sigma}^2 (s,a)}{\epsilon}}\\
  & \quad + \frac{\sigma^2 (\hat{s}, \hat{a})+\hat{\sigma}^2 (s,a)}{\epsilon} +
  \sqrt{\frac{\sigma^2 (\hat{s}, \hat{a})+\hat{\sigma}^2 (s,a)}{\epsilon} \frac{\sigma^2 (s,
  a)}{n_{\mathrm{eff}} (s, a) \epsilon_{\mathrm{KDE}}}}\\
  & = ((1 + L_{S \times A}) \|(\hat{s}, \hat{a})^{\top} - (s, a)^{\top} \|)^2
  + 2 (1 + L_{S \times A}) \left( \sqrt{\frac{\sigma^2 (s,
  a)}{n_{\mathrm{eff}} (s, a) \epsilon_{\mathrm{KDE}}}} + \sqrt{\frac{\sigma^2
  (\hat{s}, \hat{a})+\hat{\sigma}^2 (s,a)}{\epsilon}} \right) \|(\hat{s}, \hat{a})^{\top} - (s,
  a)^{\top} \|\\
  & \quad + \frac{\sigma^2 (s, a)}{n_{\mathrm{eff}} (s, a)
  \epsilon_{\mathrm{KDE}}} + \frac{\sigma^2 (\hat{s}, \hat{a})+\hat{\sigma}^2 (s,a)}{\epsilon} +
  \sqrt{\frac{\sigma^2 (\hat{s}, \hat{a})+\hat{\sigma}^2 (s,a)}{\epsilon} \frac{\sigma^2 (s,
  a)}{n_{\mathrm{eff}} (s, a) \epsilon_{\mathrm{KDE}}}}
\end{array}
\end{align}
Let \(C_3=1 + L_{S \times A}\), then 
\begin{align}
    \begin{array}{ll}
    \| \hat{\theta}_{t + 1} - \theta_{t + 1} \|&\leq 
  C_1 \|d - \hat{d} \|^2 + C_2 \|d - \hat{d} \| \\
  & \leq C_1 (\alpha, \gamma)
  (C_3 \|(\hat{s}, \hat{a})^{\top} - (s, a)^{\top} \|)^2\\
  & \quad + C_3 \left( 2 C_1 (\alpha, \gamma) \sqrt{\frac{\sigma^2 (s,
  a)}{n_{\mathrm{eff}} (s, a) \epsilon_{\mathrm{KDE}}}} + 2 C_1 (\alpha,
  \gamma) \sqrt{\frac{\sigma^2 (\hat{s}, \hat{a})+\hat{\sigma}^2 (s,a)}{\epsilon}} + C_2 (\alpha,
  \gamma) \right) \|(\hat{s}, \hat{a})^{\top} - (s, a)^{\top} \|\\
  & \quad + C_1 (\alpha, \gamma) \left( \frac{\sigma^2 (s,
  a)}{n_{\mathrm{eff}} (s, a) \epsilon_{\mathrm{KDE}}} + \frac{\sigma^2
  (\hat{s}, \hat{a})+\hat{\sigma}^2 (s,a)}{\epsilon} + \sqrt{\frac{\sigma^2 (\hat{s},
  \hat{a})+\hat{\sigma}^2 (s,a)}{\epsilon} \frac{\sigma^2 (s, a)}{n_{\mathrm{eff}} (s, a)
  \epsilon_{\mathrm{KDE}}}} \right)\\
  & \quad + C_2 (\alpha, \gamma) \left( \sqrt{\frac{\sigma^2 (s,
  a)}{n_{\mathrm{eff}} (s, a) \epsilon_{\mathrm{KDE}}}} + \sqrt{\frac{\sigma^2
  (\hat{s}, \hat{a})+\hat{\sigma}^2 (s,a)}{\epsilon}} \right)\\
  &\leq  C_1(\alpha, \gamma)C_{3}^{2}\|(s,a)^\top - (\hs,\ha)^\top\|^{2} \\
  &\quad +  C_3 \left( 2 C_1 (\alpha, \gamma) \sqrt{\frac{\sigma^2 (s,
  a)}{n_{\mathrm{eff}} (s, a) \epsilon_{\mathrm{KDE}}}} + 2 C_1 (\alpha,
  \gamma) \left(\sqrt{\frac{\sigma^2 (\hat{s}, \hat{a})}{\epsilon}}+\sqrt{\frac{\hat{\sigma}^2 (s,a)}{\epsilon}}\right) + C_2 (\alpha,
  \gamma) \right) \|(\hat{s}, \hat{a})^{\top} - (s, a)^{\top} \|\\
  & \quad + C_1 (\alpha, \gamma) \left( \frac{\sigma^2 (s,
  a)}{n_{\mathrm{eff}} (s, a) \epsilon_{\mathrm{KDE}}} + \frac{\sigma^2
  (\hat{s}, \hat{a})+\hat{\sigma}^2 (s,a)}{\epsilon} + \left(\sqrt{\frac{\sigma^2 (\hat{s},
  \hat{a})}{\epsilon}}+\sqrt{\frac{\hat{\sigma}^2 (s,a)}{\sigma}}\right) \sqrt{\frac{\sigma^2 (s, a)}{n_{\mathrm{eff}} (s, a)
  \epsilon_{\mathrm{KDE}}}} \right)\\
  & \quad + C_2 (\alpha, \gamma) \left( \sqrt{\frac{\sigma^2 (s,
  a)}{n_{\mathrm{eff}} (s, a) \epsilon_{\mathrm{KDE}}}} + \sqrt{\frac{\sigma^2 (\hat{s}, \hat{a})}{\epsilon}}+\sqrt{\frac{\hat{\sigma}^2 (s,a)}{\epsilon}}\right)\\
  &= C_1(\alpha, \gamma)C_{3}^{2}\|(s,a)^\top - (\hs,\ha)^\top\|^{2} \\
     &\quad+C_2(\alpha, \gamma)C_3\|(s,a)^\top - (\hs,\ha)^\top\|
  \\
    &\quad +\left(\mathcal{O}(\sigma(s,a))+\mathcal{O}(\sigma(\hs,\ha))+\mathcal{O}(\hat{\sigma}(s,a))\right)\|(s,a)^\top - (\hs,\ha)^\top\|\\
        &\quad+\mathcal{O}(\sigma(s,a))+\mathcal{O}(\sigma(\hs,\ha))+\mathcal{O}(\hat{\sigma}(s,a))\\
        &\quad+\mathcal{O}(\sqrt{\hat{\sigma}(s,a)\sigma(s,a)})+\mathcal{O}(\sqrt{\sigma(s,a)\sigma(\hs,\ha))}\\
    &= C_1(\alpha, \gamma)C_{3}^{2}\|(s,a)^\top - (\hs,\ha)^\top\|^{2} \\
     &\quad+C_2(\alpha, \gamma)C_3\|(s,a)^\top - (\hs,\ha)^\top\|\\
     &\quad+\mathcal{O}(\sigma_{\max})\|(s,a)^\top - (\hs,\ha)^\top\|+\mathcal{O}(\sigma_{\max}),
\end{array}
\end{align}
where \(\sigma_{\max}=\max\left(\sigma(s,a), \sigma(\hs,\ha), \hat{\sigma}(s,a), \sqrt{\hat{\sigma}(s,a)\sigma(s,a)}, \sqrt{\sigma(s,a)\sigma(\hs,\ha)}\right)\).
\end{proof}

\section{Experiment details}\label{supp:exp_details}

We use the \texttt{PyTorch} re-implementation of \ac{mbpo} by \citet{liu_re_2019} as the baseline algorithm. We added support for the \texttt{HalfCheetah-V2} environment based on the original implementation by \citet{janner_when_2019}. The hyper-parameters for the baseline algorithm are identical to those in \citet{janner_when_2019} and are shown in \Cref{tab:hyper-invped}. The explanation of symbols is provided in \Cref{tab:hyper-explain}. The definitions of ensemble size \(B\) and policy updates per environment step \(G\) can be found in Algorithm 2 of \citet{janner_when_2019}.

\begin{table}[htb]
\label{tab:hyper-explain}
\centering
\caption{Explanation of Hyper-parameters}
\begin{tabular}{cl}
\toprule
Hyperparameter & Explanation \\
\midrule
$K$ & number of epochs \\
$H$ & environment steps per epoch \\
$M$ & model rollouts per environment step \\
$B$ & ensemble size \\
$G$ & policy updates per environment step \\
$L$ & model horizon (rollout length) \\
$\epsilon$ & rejection level (if using static rejection rule) \\
\bottomrule
\end{tabular}
\end{table}

\begin{table}[htb]
\centering
\caption{Experiment Hyper-parameters.}\label{tab:hyper-invped}
\begin{tabular}{lrccccccc}
\toprule
\multicolumn{1}{c}{\multirow{2}{*}{Environment}} & \multicolumn{1}{c}{\multirow{2}{*}{Algorithm}} & \multicolumn{7}{c}{Parameters} \\
\cmidrule{3-9}
& & $K$ & $H$ & $M$ & $B$ & $G$ & $L$ & $\epsilon$ \\
\midrule
\multirow{3}{*}{InvertedPendulum-V2} & \ac{mbpo} & 20 & 250 & 400 & 7 & 20 & 1 & N/A\\
& Data Filter & 20 & 250 & 400 & 7 & 20 & 5 & Dynamic\\
& Data Filter w/o Ensemble & 20 & 250 & 400 & 1 & 20 & 5 & Dynamic \\
\midrule
\multirow{3}{*}{Walker2d-V2} & \ac{mbpo} & 150 & 1000 & 400 & 7 & 20 & 1 & N/A \\
& Data Filter & 150 & 1000 & 400 & 7 & 20 & 2 & Dynamic \\
& Data Filter w/o Ensemble & 150 & 1000 & 400 & 1 & 20 & 10 & Dynamic \\
\midrule
\multirow{3}{*}{HalfCheetah-V2} & \ac{mbpo} & 300 & 1000 & 400 & 7 & 40 & 1 & N/A\\
& Data Filter & 300 & 1000 & 400 & 7 & 40 & 10 & 5 \\
& Data Filter w/o Ensemble & 300 & 1000 & 400 & 1 & 40 & 10 & 5 \\
\bottomrule
\end{tabular}
\end{table}

The experiments are conducted on a computer with a single \texttt{NVIDIA RTX 3090} GPU and a single \texttt{Intel i7-7800X} CPU, running \texttt{Ubuntu 20.04 LTS}. 
%
% In this machine, training on \texttt{InvertedPendulum-V2} takes three hours, while training on \texttt{Walker2D-V2} takes 300 hours, and \texttt{HalfCheetah-V2} takes 150 hours. 

% \zd{Replay Buffer implementation with HNSW}

% \zd{Model Training and Policy Training details}

\end{document}